\documentclass{article}
\usepackage{ijcai26}

\usepackage{times}
\usepackage{soul}
\usepackage{url}
\usepackage[hidelinks]{hyperref}
\usepackage[utf8]{inputenc}
\usepackage[small]{caption}
\usepackage{graphicx}
\usepackage{amsmath}
\usepackage{amsthm}
\usepackage{booktabs}
\usepackage{algorithm}
\usepackage{algorithmic}
\usepackage[switch]{lineno}

\usepackage[textsize=tiny]{todonotes}
\usepackage{amsfonts}
\usepackage{subcaption}
\usepackage{wrapfig}

\newtheorem{definition}{Definition}
\newtheorem{assumption}{Assumption}

\newtheorem{theorem}{Theorem}

\title{ROAD: Adaptive Data Mixing for Offline-to-Online Reinforcement Learning via Bi-Level Optimization}

\author{
Letian Yang$^{1}$\thanks{These authors contributed equally.}
\and
Xu Liu$^{1}$\footnotemark[1]
\and
Yiqiang Lu$^{2}$
\and
Jian Liu$^{2}$
\and
Weiqiang Wang$^{2}$
\And
Shuai Li$^1$
\\
\affiliations
$^1$Shanghai Jiao Tong University, Shanghai, China\\
$^2$Ant Group, Shanghai, China
}
\begin{document}

% \renewcommand{\thefootnote}{\fnsymbol{footnote}}
% \footnotetext[2]{These authors contributed equally.}

% \maketitle
% \input{sections/0-abstract}
% \input{sections/1-introduction}
% \input{sections/2-related_work}
% \input{sections/3-theory}
% \input{sections/4-method}
% \input{sections/5-experiment}
% \input{sections/6-conclusion}

\maketitle

\begin{abstract}
    Offline-to-online reinforcement learning harnesses the stability of offline pretraining and the flexibility of online fine-tuning. A key challenge lies in the non-stationary distribution shift between offline datasets and the evolving online policy. Common approaches often rely on static mixing ratios or heuristic-based replay strategies, which lack adaptability to different environments and varying training dynamics, resulting in suboptimal tradeoff between stability and asymptotic performance.
    In this work, we propose Reinforcement Learning with Optimized Adaptive Data-mixing (ROAD), a dynamic plug-and-play framework that automates the data replay process. We identify a fundamental objective misalignment in existing approaches. To tackle this, we formulate the data selection problem as a bi-level optimization process, interpreting the data mixing strategy as a meta-decision governing the policy performance (outer-level) during online fine-tuning, while the conventional Q-learning updates operate at the inner level. To make it tractable, we propose a practical algorithm using a multi-armed bandit mechanism. This is guided by a surrogate objective approximating the bi-level gradient, which simultaneously maintains offline priors and prevents value overestimation.
    Our empirical results demonstrate that this approach consistently outperforms existing data replay methods across various datasets, eliminating the need for manual, context-specific adjustments while achieving superior stability and asymptotic performance.
\end{abstract}
\section{Introduction} \label{sec:intro}

Deep reinforcement learning (RL) has achieved remarkable access in various domains \cite{rl-survey,rl-co-survey}, yet its deployment in real-world applications is fundamentally restricted by the prohibitive sample complexity of online learning methods \cite{exploration-rl-survey,rl-sample-complexity}, especially in safety-critical scenarios (e.g., healthcare) or costly environments (e.g., autonomous driving) \cite{rl-healthcare-survey,autonomous-driving-survey}. Offline RL offers a promising alternative by leveraging previously collected datasets to extract a policy without environment interaction \cite{offline-rl-tutorial}. However, strictly offline learning is inherently limited by the quality and coverage of the static datasets, failing to reach a performance comparable to agents trained online. To bridge this gap, offline-to-online (O2O) RL has emerged as a new paradigm, aiming to initialize an agent with a safe offline policy (the offline phase) and subsequently refine it via direct online interaction (the online phase) \cite{awac}. This paradigm seeks to maintain the safety of offline policy while enabling the plasticity to explore and improve beyond the offline dataset.

Under such a twofold objective, effective O2O training is non-trivial. The transition from offline pretraining to online fine-tuning introduces severe distribution shift \cite{br-pessimism}. Naively concatenating offline and online RL algorithms typically results in a catastrophic ``unlearning'' phenomenon, characterized by a significant performance dip at the beginning of online fine-tuning \cite{cal-ql}.

Prior works have addressed the distribution shift primarily through two complementary perspectives: Q-value calibration and data-level distribution alignment. Works on the former aspect \cite{awac,cal-ql} are aimed for a balance between pessimism from offline RL \cite{cql,iql} and optimism required for online RL \cite{sac}. By contrast, the latter aspect focuses on the data-level distribution shift when transitioning from the behavior policy of offline dataset to the offline-trained policy. Existing data-centric solutions largely rely on static strategies \cite{cal-ql,pex} or fixed heuristics \cite{br-pessimism,a3rl} to mix offline and online dataset. However, we argue that there is no universally optimal data mixing strategy. Instead, the optimal data replay pattern is tailored to the quality of the offline data, the environment properties and even the stage of training. For instance, a high proportion of offline data might be necessary in early training to stabilize updates, whereas later stages might benefit from prioritizing online experience to correct the bias from offline data.

In this paper, we investigate the mechanism by which data mixing strategies influence online fine-tuning, characterizing data mixing as a reshaping of the Q-learning loss landscape. From this perspective, we identify an objective misalignment problem in offline-to-online settings, where the conventional Q-learning objective of minimizing the Bellman error does not suffice to guarantee RL performance.

To address this objective misalignment issue, we propose \textbf{R}einforcement Learning with \textbf{O}ptimized \textbf{A}daptive \textbf{D}ata-mixing (ROAD), a dynamic, plug-and-play framework that adaptively determines the optimal data mixing strategy. We fundamentally reframe the data selection problem as a bi-level optimization process in the functional space of data mixing strategies, where the inner-level updates follow standard Q-learning, and the outer-level takes data mixing as a meta-decision problem to maximize the expected performance of the online policy, thereby aligning Q-learning with the ultimate goal of maximizing RL performance. A solution is provided by computing the gradient of the outer-level objective.

To further deal with the overestimation bias of the outer-level objective and thus bridge the gap between theory and practice, ROAD constructs a surrogate as an unbiased estimator for the outer-level objective. A practical algorithm utilizing a sliding-window multi-armed bandit (MAB) to optimize the data mixing ratio is proposed to optimize this surrogate objective. Intuitively, ROAD selects an arm that maximizes the performance increment in trusted regions where the uncertainty of the Q-function estimator is negligible and mitigates training instability by suppressing value overestimation.

In summary, our contributions are as follows:

\begin{itemize}
    \item We identify an objective misalignment problem in data-centric O2O RL and solve this misalignment by formulating it as a bi-level optimization, where the data-mixing strategy is considered a meta-decision variable.
    \item We propose ROAD, a plug-and-play framework that approximates the bi-level optimization theory with a computable, unbiased surrogate objective and a MAB mechanism optimizing data mixing ratios in real-time.
    \item We evaluate ROAD across extensive benchmarks and the results show that ROAD outperforms static and heuristic-based data selection strategies, ensuring training stability and asymptotic performance.
\end{itemize}
\section{Related Works} 
\label{sec:related_works}

\subsection{Offline-to-Online RL}

The primary challenge of O2O RL lies in the mismatch between the principles underlying the two phases. Specifically, offline RL algorithms feature restrictive policy constraints and pessimistic Q-value estimation \cite{cql,iql} to ensure conservatism, which inherently limits the exploration capabilities during online fine-tuning \cite{policy-finetuning,hybrid-rl,hybrid-ftpl}, resulting in slow performance gains.

To mitigate this, transitional frameworks between phases have been proposed. Some algorithms have adopted implicit or dynamic constraints \cite{awac,adaptive-bc,br-pessimism} which naturally loosen in the online phase. Similarly, Cal-QL \cite{cal-ql} conducts offline Q-value estimation in a calibrated rather than pessimistic manner; WSRL \cite{wsrl} introduces a separate warm-start phase to calibrate Q-value and then discards the offline dataset.

Apart from balancing pessimism and optimism, one line of work embraces the discrepancy between offline and online settings. APL \cite{apl} employs pessimistic updates on offline data and optimistic updates on online data; PEX \cite{pex} maintains separate policies for offline and online learning; MOORL \cite{moorl} treats the O2O RL problem as a meta-learning problem with two distinct tasks, offline and online learning.

\subsection{Offline Data Replay for Offline-to-Online RL}

While algorithmic consistency is crucial, the strategy to address data-level distribution shift plays a decisive role in ensuring stability and plasticity in the online fine-tuning phase. Although some argue that offline data can be completely discarded in online learning \cite{wsrl}, its adaptiveness to various environments is questioned \cite{three-regimes}. The predominant solution to offline data utilization during online interaction is to replay offline data in the online replay buffer.

\paragraph{Static Mixing Strategies} Early attempts directly initialize the replay buffer with offline dataset \cite{early-dqfd,early-dqfd-robotics}, which proves unstable \cite{modem}. Instead, the current mainstream is to sample independently from two data sources with a mixing ratio. QT-Opt gradually increase the online data ratio from $1\%$ to $50\%$ \cite{qt-opt}. Cal-QL, RLPD and PROTO implement symmetric sampling (a fixed $50\%$-$50\%$ split between offline and online data) \cite{cal-ql,rlpd,proto}, which is also adopted in concurrent researches \cite{mixing-app-simlauncher,q-chunking}. While simple, these static strategies are inherently suboptimal for lack of context awareness and environment adaptiveness.

\paragraph{Heuristic-based Data Selection} Recognizing the limitations of fixed mixing ratios, there have been works toward selective sampling. Balanced replay prioritizes ``near on-policy'' offline samples with density ratio estimation \cite{br-pessimism}. Starting from the same heuristic, A3RL filter for samples that are also advantage-aligned \cite{a3rl} and ARB propose a learning-free method to evaluate the ``on-policyness'' \cite{arb}. Despite mitigating distribution shift, near on-policy selection heuristics reduces data selection to a form of online data augmentation, limiting the algorithm's ability to fully leverage the environmental information encoded in the offline dataset. Our empirical findings, together with prior analytical works \cite{three-regimes}, indicate that no single sampling strategy generalizes universally, which motivates an adaptive, context-sensitive strategy that adapts to learning dynamics and environmental feedback \cite{cal-ql,apl}.
\section{Bi-Level Optimization for O2O RL}
\label{sec:online-as-bi-level-opt}

In this section, we formulate data mixing as an optimization problem. From a functional perspective, we offer a novel perspective on how different data mixing strategies influence the training dynamics. Under this perspective, we argue that an optimal mixing strategy should be derived by aligning the Q-learning objective with the true RL objective. Based on this formulation, we present an idealized optimization algorithm as a theoretical solution.

\subsection{Objective Misalignment in O2O Fine-Tuning}
\label{sec:theory-preliminary}

We formulate the O2O RL problem within the Markov Decision Process (MDP) framework $\mathcal{M} = \left( \mathcal{S}, \mathcal{A}, p, r, \gamma \right)$. The online fine-tuning begins with a static offline dataset $\mathcal{D}^{\text{offline}}$ generated by a behavior policy $\pi_\beta$ and an offline initialization for the policy $\pi^{\text{off}}$ and the Q-function $f^{\text{off}}$ in a function class $\mathcal{F}$. Let the initial state distribution be $\mu \in \Delta(\mathcal{S})$ and $\rho$ be any trajectory, then the goal is to maximize the cumulative value function, denoted by $J(\pi) = \mathbb{E}_{s \sim \mu} \left[ V^{\pi}(s) \right] = \mathbb{E}_{\rho\sim\pi, s \sim \mu} \left[ \sum_{t=0}^{\infty} \gamma^t r_t \vert s_0=s \right]$.

\paragraph{Online fine-tuning via fitted Q-iteration.}
Standard value-based online fine-tuning algorithms typically follow a fitted Q-iteration (FQI) scheme \cite{cal-ql}. At each iteration $k$, the agent collects an online data batch $\mathcal{D}^{k}$ with the online policy $\pi^{k}$. Then the Q-value iteration is performed on a mixed distribution $d^{\text{sample}}$ built with both the offline dataset and online replay buffer. Specifically, the Q-function updates with Equation \eqref{eq:fqi-update} where $\mathcal{T}$ is the Bellman operator defined as $\mathcal{T}f(s,a) = r(s,a) + \gamma \mathbb{E}_{s'} \left[ \max_{a'} f(s',a') \right]$.
\begin{align}
    f^{k+1} \leftarrow \arg\min_{f\in\mathcal{F}} \left\{ \mathbb{E}_{d^{\text{sample}}} \left[ \mathcal{T}f^k(s,a) - f(s,a) \right]^2 \right\}.
    \label{eq:fqi-update}
\end{align}
Crucially, the construction of the mixed distribution is formulated via a data mixing strategy $g\in \mathcal{G}$, defined as a mapping from online and offline data sources to a sampling distribution, where $\mathcal{G}$ is the function space of all such strategies.
\begin{align}
    d^{\text{sample}}_g = g(\mathcal{D}^{\text{offline}}, \mathcal{D}^{1},...,\mathcal{D}^{k}).
    \label{eq:data-mixing-definition}
\end{align}

\paragraph{The loss landscape perspective on data mixing.}
Equation \eqref{eq:fqi-update} reveals a critical insight that Q-learning is essentially projecting the Bellman target onto the representable function space $\mathcal{F}$, i.e., $f^{k+1} \leftarrow \Pi_{\mathcal{F}, d_g^{\text{sample}}} \left( \mathcal{T} f^{k} \right)$. Here $\Pi_{\mathcal{F}, d_g^{\text{sample}}}$ denotes the projection operator defined by weighted $L_2$-norm $\left\Vert \cdot \right\Vert_{d_g^{\text{sample}}}$. This formulation makes explicit the mechanism through which data mixing strategies affect online RL fine-tuning, as depicted below. The data mixing strategy governs the geometry of the projection. By altering the sampling density over the state-action space, different mixing strategies define distinct loss landscapes, which further steers the optimization toward different optima $f^*(g)$. The induced variation in Q-function updates further yields different online policies and, ultimately, different RL performance levels $J(\pi)$.
\begin{align*}
    g \rightarrow \text{loss landscape} \rightarrow f^* \rightarrow \pi \rightarrow J(\pi).
\end{align*}

\paragraph{The misalignment between objectives.}
The dependency chain of variables above demonstrates that the primary obstacle to identifying ideal data mixing stems from its \textit{indirect} influence on the training dynamics, which gives rise to an objective misalignment problem. Specifically, contrary to standard online RL, minimizing the Bellman error in online fine-tuning does not guarantee maximizing the policy return due to the extra variable $g$.

Existing approaches dependent on static mixing ratios and heuristic selection effectively lock the projection geometry based on fixed priors, assuming that a particular landscape is universally beneficial. However, this assumption is fundamentally flawed because the relationship between the projection geometry and online performance is dynamic, dependent on the available data, the online policy and the stage of training.
As an example, geometry required for early-stage stability (anchoring to offline data to prevent catastrophic forgetting) is the inverse of that required for asymptotic plasticity (emphasizing online samples to correct offline bias).
The fixed priors cut off the explicit feedback from the online policy performance, failing to satisfy these mutually exclusive requirements simultaneously.

To resolve this misalignment, we transcend static rules and treat $g$ as a meta-decision variable to study its effect on the policy performance (our final objective) rigorously. Since aligning the Bellman objective with RL performance necessitates precise control over the loss landscape across the state-action space, we abstain from parameterizing $g$ and formulate it directly as a functional meta-decision variable.

\subsection{The Bi-Level Optimization Formulation and The Theoretical Solution}
\label{sec:problem-formulation}

To address the objective misalignment problem, we reframe data mixing as a dynamic meta-decision process, solving an optimization problem in the space of mixing strategies to dynamically align the Q-learning objective with policy performance maximization. Notably, this optimization is inherently hierarchical because the final objective (policy performance) depends on the learned Q-function, which is itself an optimum for minimizing the Bellman error. The constraint from an optimal inner Q-learning loop naturally casts the data mixing problem into the framework of bi-level optimization.

Formally, the inner loop is defined as standard Q-learning updates, which we treat as a fixed subroutine. The objective $\mathcal{J}^{k}_{\text{in}}(g,f)$ is exactly the Bellman error to be minimized on the mixed dataset, as is defined in Equation \eqref{eq:fqi-update}. In order to dynamically determine the optimal mixing, the outer-loop is expected to make a decision on the meta-variable $g$ within each FQI iteration, maximizing the expected return $J(\pi)$ of the induced policy. Using the chain of variable dependency, we can express the objective as a function of $g$, leading to the objective in Equation \eqref{eq:bi-level-objective}.
\begin{align}
    & \max_{g \in \mathcal{G}} \mathcal{J}^{k}_{\text{out}}(g)=\max_{g \in \mathcal{G}} J \left( \pi \left( \arg\min_{f} \mathcal{J}^{k}_{\text{in}}(g,f)) \right) \right) \nonumber \\
    & \text{where } \mathcal{J}^{k}_{\text{in}}(g,f) = \mathbb{E}_{d_g^{\text{sample}}} \left[ \mathcal{T}f^k(s,a) - f(s,a) \right]^2.
    \label{eq:bi-level-objective}
\end{align}
Here we assume a differentiable policy mapping from the inner-level optimal $f^*$ to the behavior policy $\pi(f^*)$ to ensure that the gradient exists.

\paragraph{The theoretical solution.}
Inspired by recent breakthroughs in bi-level optimization theory \cite{bi-level-opt}, we derive a theoretical solution to the proposed setting. Within each FQI iteration, we solve the optimization via gradient ascent.
The critical part is the gradient calculation of the outer-level objective demonstrated in Equation \eqref{eq:outer-level-grad}. A full derivation and the explanation are provided in Appendix \ref{sec:appendix-theoretical-solution}.
\begin{align}
    \nabla_g \mathcal{J}_{\text{out}} (g) = - \int_{\mathcal{S}\times\mathcal{A}} w_g(s,a) \partial_{g,f} \mathcal{J}_{\text{in}}(s,a) \mathrm{d}s \mathrm{d}a.
    \label{eq:outer-level-grad}
\end{align}
Here $w_g$ denotes the density ratio between the distribution of a policy induced by the online policy and its advantage function, and the sampling distribution; the mixed Fréchet derivative $\partial_{g,f} \mathcal{J}_{\text{in}}$ evaluates the sensitivity of Bellman error to the data mixing strategy. Intuitively, Equation \eqref{eq:outer-level-grad} implies that the ideal $g^*$ induces a weighted on-policy distribution biased towards the loss-sensitive samples. Further discussion is provided in Appendix \ref{sec:appendix-interpret-theoretical-gradient}.
\section{ROAD: Towards Practical RL Fine-Tuning}
\label{sec:road}

\begin{figure*}[t]
\begin{center}
\includegraphics[width=0.9\textwidth]{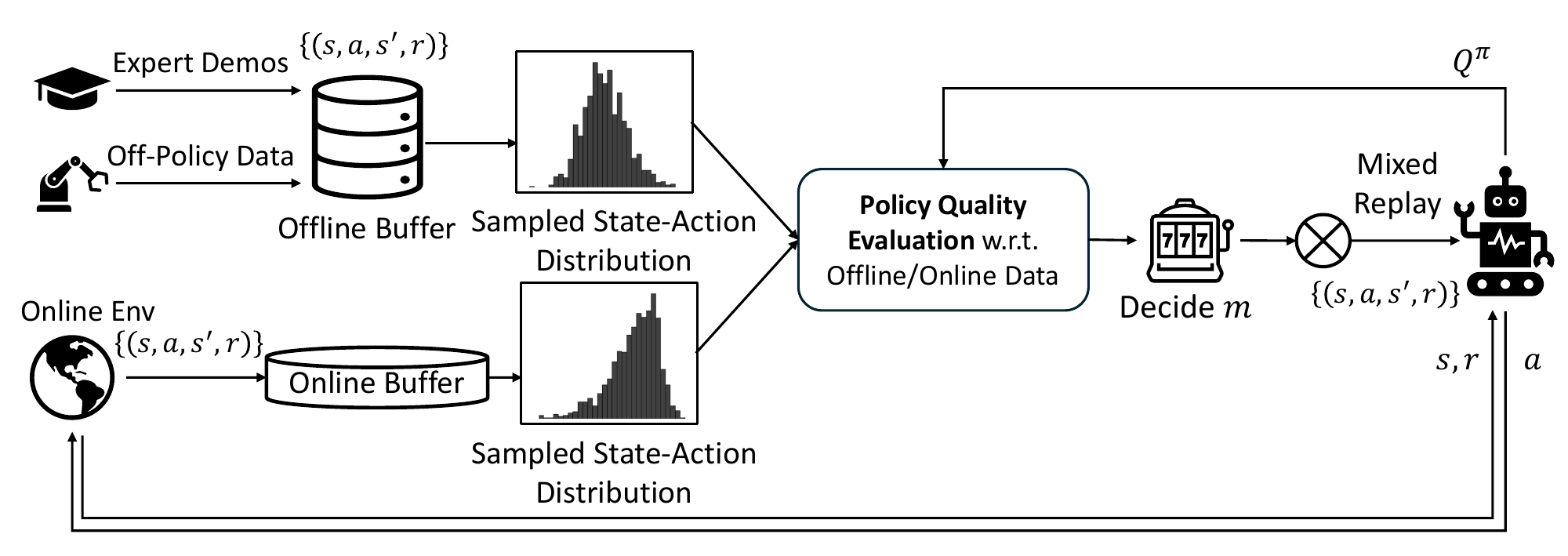}
\end{center}
\caption{Illustration of the online training scheme of ROAD. In the online training phase, ROAD controls the offline data replay pattern by selecting the mixing ratio in an adaptive fashion with the derived surrogate objective for the bi-level optimization problem. The mixing ratio $m$ adjusted by ROAD controls the mixed replay for the offline dataset and the online buffer.}
\label{Fig: framework}
\end{figure*}

Despite an idealized procedure for optimal data mixing strategy in Section \ref{sec:problem-formulation}, directly solving for an optimum with Equation \eqref{eq:outer-level-grad} is intractable in online settings. Consequently, in this section, we present a practical algorithm for the bi-level optimization, ROAD, for online RL fine-tuning.

\subsection{Bridging Theory with Practice}
\label{sec:bridge-theory-with-practice}

The primary obstacle to implementing theoretical gradient ascent problem are the \textit{estimation bias} of the Q-function and severe \textit{numerical instabilities} in gradient computation. First, in practice, $J(\pi(f^*))$ is a biased estimation for the outer-level objective because $\pi$ is optimized against a Q-function $f^*$ characterized by uncertainty. Maximizing this biased objective inevitably misguides the optimization process, which necessitates an unbiased estimator of $J(\pi)$. Furthermore, computing this gradient involves multiple density estimations including that for a rapidly evolving online policy $d^{\pi}$.
These estimations are inherently unstable on online data, and the arithmetic operations further aggravate this instability. Meanwhile, the high variance associated with the mixed Fréchet derivative term easily mislead the gradient direction in sampling-based estimation.
Consequently, to make unbiased estimation on the outer-level objective and meanwhile ensure computational feasibility, we adopt a zero-order optimization approach, employing a MAB mechanism to search for an optimal data mixing strategy.
We establish our practical algorithm on the following two key assumptions.
\begin{assumption}
    The offline dataset offers sufficient number of samples within its support, leading to negligible uncertainty.
    \label{assumption:offline-uncertainty}
\end{assumption}
\begin{assumption}
    The training dynamics varies slowly.
    \label{assumption:slow-vary}
\end{assumption}

Assumption \ref{assumption:offline-uncertainty} posits that offline data serves as an anchor. Within this ``trust region'', it is assumed that uncertainty does not significantly distort the perceived relative value of actions. Meanwhile, Assumption \ref{assumption:slow-vary} is the fundamental premise for modeling non-stationary environments using MAB. It allows for a sliding-window approach to evaluate the performance of data mixing strategies.

\paragraph{An unbiased objective estimator.}
An unbiased estimation for the outer-level objective is shown in Equation \eqref{eq:unbiased-estimation}, where the empirical estimation error $\epsilon = \hat{f} - f$ is assumed to be a zero-mean random field, effectively measuring the epistemic uncertainty. The specific formulation of the bias term is provided in Appendix \ref{sec:build-surrogate}.
\begin{align}
    \mathcal{J}^{k}_{\text{out}}(g) = \mathbb{E}_{(s,a) \sim d^{\pi}} \left[ \hat{f}^*(s,a) \right] - \mathbb{E}_{\epsilon} \left[ \text{Bias}_{\hat{f}^*}\right]
    \label{eq:unbiased-estimation}
\end{align}

\paragraph{Linearized mixing strategies.}
To tackle computational intractability, we restrict the search space to linear combinations of offline and online distributions, parameterized by an offline data replay ratio $m \in [0,1]$. Then the mixed data distribution is $d_m^{\text{sample}}(s,a) = m d^{\text{off}}(s,a) + (1-m) d^{\text{on}}(s,a)$. To ensure compatibility with the MAB search mechanism, the candidate set of mixing ratios $\Lambda$ is a fixed discrete one.

\subsection{The Practical Algorithm}
\label{sec:road-practical}

In this section we construct a robust surrogate objective utilized as the reward within a sliding-window MAB framework. Our approach relies exclusively on readily available statistics and avoids complex auxiliary techniques such as density estimators or Q-ensembles \cite{br-pessimism}, incurring virtually no additional computational cost.

\paragraph{The surrogate objective.}
By approximating the unbiased outer-level objective with statistics directly from the training process, we formulate our surrogate objective as Equation \eqref{eq:R_q}, where $\kappa$ is a hyperparameter. Here we have replaced the theoretical notations with the practical ones, where $Q_{\phi}$ and $\pi_\theta$ are the Q-function and the policy parameterized by $\phi$ and $\theta$ respectively. In this surrogate, $\Delta_{\text{off}}$ is the empirical estimator of RL performance improvement with respect to the offline behavior policy. By employing this advantage-like mechanism, we reduce the variance of performance estimation, leading to a stable MAB reward. Meanwhile, $\Delta_{\text{on}}$ is a theoretically grounded approximation of the overestimation bias at early stages, stabilizing the training process. The justification is provided in Appendix \ref{sec:build-surrogate}.
\begin{align}
    & R_q = \Delta_{\text{off}} - \kappa \Delta_{\text{on}} \text{ where } \nonumber\\
    & \begin{cases}
        \Delta_{\text{off}} = \mathbb{E}_{s \sim \mathcal{D}^{\text{offline}},a\sim \pi_\theta} \left[  Q_{\phi}(s,a) \right] - \mathbb{E}_{(s,a) \sim \mathcal{D}^{\text{offline}}} \left[ Q_{\phi}(s,a) \right] \\
        \Delta_{\text{on}} = \mathbb{E}_{s \sim \mathcal{D}^{\text{online}}, a\sim\pi_\theta} \left[ Q_{\phi}(s,a) \right] - \mathbb{E}_{(s,a)\sim\mathcal{D}^{\text{online}}} \left[ Q_{\phi}(s,a) \right].
    \end{cases}
    \label{eq:R_q}
\end{align}
The two terms correspond to approximations of the empirical performance and the uncertainty penalty respectively. Maximizing $\Delta_{\text{off}}$ encourages policy improvement in the offline ``trust region'', while minimizing $\Delta_{\text{on}}$ stabilizes online learning and prevents unlearning. Though seemingly counter-intuitive, $\Delta_{\text{on}}$ functions analogously to a min-max regularization by constraining aggressive policy updates and preventing the policy from exploiting the Q-function bias.

\paragraph{Implementation of ROAD.} With a practical surrogate objective, we describe our algorithm ROAD as follows. ROAD works under a certain period (e.g., an episode or a certain number of steps) and maintains a recording list for the arm pulled and the surrogate objective $(m^k, R_q^k)$ for each period $k$. When a period begins, ROAD selects a mixing ratio by maximizing the upper confidence bound (UCB) as is shown in Equation \eqref{eq:mab_ucb} and fixes the ratio during the period. In order to model the non-stationarity in the training dynamics, ROAD adopts a sliding window of size $\tau$, which means that we select the best arm using only results from the recent $\tau$ periods.
\begin{align}
    m^{k} = \arg\max_{m \in \Lambda} \left\{ \bar{R}_{q,m} + \sqrt{\frac{c \log (k \wedge \tau)}{N_k(\tau,m)}}\right\},
    \label{eq:mab_ucb}
\end{align}
Here $\bar{R}_{q,m}$ is the average of recorded $R_q$ values for mixing ratio $m$ within the sliding window; $N_k(\tau, m)$ denotes the total number of times $m$ has been selected up to period $k$.
\section{Experiment} \label{sec 5}

\begin{table*}[t]
\centering
\scriptsize
\setlength{\tabcolsep}{2.5pt}
\caption{Performance of ROAD and the baseline methods for offline data replay under Antmaze tasks, Locomotion tasks and Kitchen tasks with IQL. All results are assessed across 4 random seeds. The best performance for fixed mixing ratios is \underline{underlined}, and the best-performed score is \textbf{bolded}.\\ \ }
\label{tab:performance_metrics2}
\begin{tabular}{lcccccccccc}
\toprule
\multicolumn{1}{c}{} & \multicolumn{6}{c}{Fixed Mixing Ratios} & \multicolumn{1}{c}{} & \multicolumn{1}{c}{} & \multicolumn{1}{c}{} & \multicolumn{1}{c}{} \\
\cmidrule(lr){2-7}
Tasks & $0.0$ & 0.1 & 0.2 & 0.3 & 0.4 & 0.5 & Uniform & Decreasing & BR & ROAD \\
\midrule
antmaze-large-diverse & 56.98$\pm$3.33 & \underline{60.32$\pm$4.08} & 55.98$\pm$2.26 & 56.01$\pm$2.29 & 50.66$\pm$3.79 & 50.35$\pm$0.86 & 60.83$\pm$4.76 & 52.15$\pm$0.25 & 46.83$\pm$6.74 & \textbf{63.13$\pm$2.24} \\
antmaze-large-play    & \underline{58.16$\pm$2.63} & 46.51$\pm$1.63 & 53.16$\pm$5.70 & 53.66$\pm$0.61 & 46.82$\pm$2.02 & 55.12$\pm$3.63 & 46.32$\pm$9.76 & 52.17$\pm$1.28 & 38.34$\pm$1.24 & \textbf{59.75$\pm$2.97} \\
antmaze-medium-diverse& 80.00$\pm$4.55 & 78.68$\pm$4.48 & 80.50$\pm$5.00 & 81.66$\pm$2.01 & 82.67$\pm$1.85 & \underline{83.29$\pm$3.71} & 82.30$\pm$2.01 & 81.82$\pm$2.52 & 81.84$\pm$0.75 & \textbf{83.67$\pm$1.39} \\
antmaze-medium-play   & \underline{82.50$\pm$3.11} & 79.17$\pm$2.40 & 77.32$\pm$1.92 & 80.16$\pm$1.66 & 82.00$\pm$2.95 & 78.62$\pm$3.53 & 82.67$\pm$0.26 & 77.16$\pm$3.02 & 80.99$\pm$0.99 & \textbf{83.49$\pm$3.62} \\
antmaze-umaze-diverse & 63.49$\pm$13.79 & \underline{69.66$\pm$15.78} & 7.83$\pm$2.24 & 38.50$\pm$33.15 & 46.51$\pm$21.76 & 16.13$\pm$8.51 & 25.65$\pm$10.24 & 48.82$\pm$23.49 & \textbf{80.66$\pm$3.99} & 72.12$\pm$23.29 \\
antmaze-umaze         & 92.16$\pm$0.94 & 93.32$\pm$0.96 & 93.16$\pm$0.85 & 93.34$\pm$1.32 & 94.33$\pm$0.47 & \underline{94.37$\pm$0.24} & 91.32$\pm$1.76 & 92.99$\pm$0.50 & 94.33$\pm$0.75 & \textbf{95.83$\pm$0.30} \\
% \textbf{Antmaze Total}& 0.00 & 0.00 & 0.00 & 0.00 & 0.00 & 0.00 & 0.00 & 0.00 & 0.00 & 0.00 \\
\midrule
halfcheetah-random            & 41.48$\pm$2.21         & 43.50$\pm$6.28         & \underline{48.25$\pm$3.30} & 39.62$\pm$3.10         & 43.71$\pm$6.07         & 44.39$\pm$4.37         & 45.91$\pm$0.06        & 42.72$\pm$0.22        & 47.43$\pm$1.12         & \textbf{49.37$\pm$3.57} \\
halfcheetah-medium-replay    & 50.70$\pm$3.38         & \underline{54.54$\pm$1.81} & 52.55$\pm$1.72         & 51.55$\pm$1.90         & 50.11$\pm$1.14         & 47.41$\pm$0.34         & 50.07$\pm$0.05        & 53.55$\pm$3.13        & 49.28$\pm$2.29         & \textbf{55.82$\pm$1.07} \\
halfcheetah-medium           & 69.61$\pm$1.58         & \underline{73.53$\pm$1.74} & 72.23$\pm$1.72         & 69.47$\pm$4.25         & 69.39$\pm$2.61         & 67.87$\pm$1.08         & 67.56$\pm$4.61        & 70.46$\pm$0.30        & 72.49$\pm$1.51         & \textbf{74.57$\pm$1.95} \\
halfcheetah-medium-expert    & 63.75$\pm$5.65         & 92.75$\pm$1.40         & 92.45$\pm$1.48         & \underline{94.52$\pm$1.13} & 93.07$\pm$1.16         & 91.54$\pm$5.56         & 92.83$\pm$1.18        & 93.88$\pm$1.70        & 93.34$\pm$2.33         & \textbf{95.06$\pm$0.55} \\
halfcheetah-expert           & 78.28$\pm$3.58         & 94.17$\pm$1.65         & 94.64$\pm$0.66         & 95.65$\pm$0.36         & 95.94$\pm$0.47         & \underline{96.61$\pm$0.45} & 93.83$\pm$1.76        & 93.15$\pm$0.32        & 94.91$\pm$0.77         & \textbf{96.86$\pm$0.39} \\
% \textbf{Halfcheetah Total}& 0.00 & 0.00 & 0.00 & 0.00 & 0.00 & 0.00 & 0.00 & 0.00 & 0.00 & 0.00 \\
\midrule
walker2d-random                  & 6.57$\pm$1.50    & 8.72$\pm$2.36   & 8.13$\pm$0.38   & 10.00$\pm$1.52  & 9.20$\pm$1.08    & \underline{10.65$\pm$1.66} & 8.58$\pm$0.31   & 9.56$\pm$1.00   & 7.58$\pm$0.86   & \textbf{12.43$\pm$0.69} \\
walker2d-medium-replay           & 46.55$\pm$24.38  & 34.56$\pm$7.61  & 34.94$\pm$3.47  & \underline{63.26$\pm$13.31} & 43.33$\pm$9.41   & 40.98$\pm$20.32 & 57.70$\pm$15.01 & 39.16$\pm$2.18  & 70.26$\pm$7.95  & \textbf{78.15$\pm$15.74} \\
walker2d-medium                   & 50.28$\pm$7.17   & \underline{59.07$\pm$14.22} & 56.72$\pm$9.78  & 44.59$\pm$9.07  & 55.74$\pm$16.19  & 31.52$\pm$8.42  & 39.78$\pm$1.67  & \textbf{64.30$\pm$10.13} & 40.11$\pm$19.46 & 60.17$\pm$8.51 \\
walker2d-medium-expert            & \underline{79.40$\pm$11.42} & 76.15$\pm$13.66 & 68.21$\pm$6.43  & 73.54$\pm$11.20 & 62.06$\pm$28.48  & 60.48$\pm$4.65  & 59.12$\pm$22.35 & 82.93$\pm$7.30  & 81.88$\pm$6.36  & \textbf{83.60$\pm$13.06} \\
walker2d-expert                   & 86.02$\pm$12.41  & \textbf{\underline{89.75$\pm$10.80}} & 41.59$\pm$23.56 & 56.49$\pm$18.25 & 73.62$\pm$9.23  & 58.19$\pm$9.67  & 52.01$\pm$8.22  & 74.28$\pm$6.07  & 79.64$\pm$2.37  & 88.06$\pm$2.09 \\
% \textbf{Walker2d Total}& 0.00 & 0.00 & 0.00 & 0.00 & 0.00 & 0.00 & 0.00 & 0.00 & 0.00 & 0.00 \\
\midrule
hopper-random         & 17.43$\pm$10.05 & 18.54$\pm$6.79 & \underline{19.71$\pm$2.87} & 12.82$\pm$3.29 & 13.51$\pm$5.57 & 15.51$\pm$5.41 & 13.77$\pm$3.61 & 19.80$\pm$1.37 & 21.48$\pm$2.84 & \textbf{31.77$\pm$4.44} \\
hopper-medium-replay  & 54.46$\pm$29.79 & \underline{80.14$\pm$25.89} & 67.24$\pm$13.88 & 71.00$\pm$25.77 & 67.35$\pm$9.05 & 77.74$\pm$29.92 & 97.93$\pm$13.56 & 77.12$\pm$9.42 & 88.27$\pm$24.20 & \textbf{99.14$\pm$9.97} \\
hopper-medium         & 87.25$\pm$19.06 & 81.84$\pm$29.83 & 61.15$\pm$18.09 & 77.46$\pm$18.12 & \underline{87.85$\pm$7.86} & 73.30$\pm$22.99 & 86.17$\pm$2.53 & 85.37$\pm$13.13 & 88.56$\pm$20.03 & \textbf{99.23$\pm$1.05} \\
hopper-medium-expert  & 46.63$\pm$4.13 & 49.86$\pm$8.83 & 47.01$\pm$12.77 & \underline{75.97$\pm$29.32} & 57.68$\pm$12.16 & 73.46$\pm$22.75 & 66.02$\pm$23.94 & 61.76$\pm$4.93 & 66.21$\pm$4.65 & \textbf{94.08$\pm$16.78} \\
hopper-expert         & 71.28$\pm$24.84 & \underline{72.74$\pm$12.96} & 62.33$\pm$14.68 & 57.80$\pm$3.67 & 52.69$\pm$10.61 & 54.07$\pm$1.54 & 50.81$\pm$5.70 & 58.80$\pm$4.32 & 39.62$\pm$5.91 & \textbf{85.10$\pm$4.27} \\
% \textbf{Hopper Total}  & 0.00 & 0.00 & 0.00 & 0.00 & 0.00 & 0.00 & 0.00 & 0.00 & 0.00 & 0.00 \\
\midrule
kitchen-partial       & 4.16$\pm$13.27 & 38.17$\pm$11.61 & 39.55$\pm$4.71 & 23.76$\pm$9.43 & 18.74$\pm$14.94 & \underline{41.87$\pm$18.71} & 22.04$\pm$0.00 & 32.93$\pm$28.74 & 29.99$\pm$29.37 & \textbf{42.65$\pm$5.13} \\
kitchen-mixed         & 0.41$\pm$0.59 & 44.34$\pm$1.58 & \underline{45.44$\pm$8.12} & 40.02$\pm$7.12 & 39.14$\pm$15.11 & 45.00$\pm$9.17 & 55.83$\pm$11.90 & 29.58$\pm$13.81 & 47.56$\pm$0.64 & \textbf{56.62$\pm$6.18} \\
kitchen-complete      & 10.04$\pm$1.56 & 51.65$\pm$10.22 & \underline{52.07$\pm$6.64} & 47.08$\pm$12.38 & 19.59$\pm$24.76 & 41.86$\pm$11.94 & 53.75$\pm$13.69 & 21.65$\pm$16.24 & 41.64$\pm$27.49 & \textbf{66.25$\pm$17.63} \\
% \textbf{Kitchen Total}& 0.00 & 0.00 & 0.00 & 0.00 & 0.00 & 0.00 & 0.00 & 0.00 & 0.00 & 0.00 \\
\midrule
\textbf{Average} & 54.07 & 62.15 & 57.28 & 58.66 & 56.49 & 56.26 & 58.45 & 59.00 & 61.80 & \textbf{71.95} \\
\bottomrule
\end{tabular}
% \vspace{0.2cm}
\end{table*}

We study how well ROAD can facilitate efficient training in the online phase with adaptive data mixing. To this end, we compare ROAD with several other data replay methods on a variety of offline-to-online RL benchmark tasks with various offline datasets. We also study the behavior of ROAD in adaptively adjusting its data replay pattern in various environments. Finally, we perform empirical studies to investigate the robustness of ROAD under different parameter settings.

\paragraph{O2O RL tasks and settings.}
Our benchmarking tasks and datasets include: \textbf{(1)} AntMaze tasks from D4RL \cite{d4rl} which involves navigating an ant quadruped robot to a predetermined goal location within a maze; \textbf{(2)} Gym MuJoCo Locomotion tasks \cite{mujoco,iql} that leverage the MuJoCo physics simulation engine targeted for robotic control with distinct morphologies and objectives; \textbf{(3)} FrankaKitchen tasks from D4RL \cite{cal-ql} where a 9-DoF Franka robot is controlled to reach a target configuration within a kitchen setup. Further information is provided in Appendix~\ref{appendix: environment details}.

\paragraph{Implementation details for ROAD.}
We examine the effect of ROAD and the baseline data replay approaches with multiple popular offline-to-online RL algorithms, including IQL \cite{iql}, PEX \cite{pex}, CQL \cite{cql}, and Cal-QL \cite{cal-ql}, to illustrate the robustness of ROAD when integrated with various offline phase implementations. Due to the space limits, we report only the IQL-based results in the main text (Section \ref{sec: 5.1}) and other results are left in Appendix \ref{Appendix: result of Cal-QL and CQL}. The candidate set $\Lambda$ is $\{0.1, 0.2, 0.3, 0.4, 0.5\}$, and the hyperparameter $\kappa = 1$. The exploration parameter $c=2.0$ in all cases to encourage adaptation to the varying training dynamics. The sliding-window size is $\tau = 1,000$ for all tasks except $\tau = k$ for Antmaze because these tasks feature sparse rewards which leads to high variance in estimating $R_q$. Further implementation details and hyper-parameters are listed in Appendix \ref{appendix: params}.

\paragraph{Baselines}

We compare ROAD with the following baseline data mixing strategies: \textbf{(i)} \textbf{\texttt{Fixed}} \cite{pex,cal-ql} mixing ratios conduct online fine-tuning on a linear combination of online and offline distributions. We run the experiments across all fixed ratios in $\Lambda$ to compare ROAD against the best one. \textbf{(ii)} \textbf{\texttt{Decreasing}} ratio adopts the heuristic that offline data stabilizes training at early stages, and thus decreases the mixing ratio linearly from $0.5$ linearly to $0.1$ (we do not take $0$ as minimum for better performance). \textbf{(iii)} \textbf{\texttt{Balanced Replay}} (\textbf{\texttt{BR}}) \cite{br-pessimism} adopts another family of heuristics that ``nearly on-policy'' samples help stabilize training.  \textbf{(iv)} \textbf{\texttt{Uniform}} ratios select mixing ratios uniformly within the candidate set.

\subsection{Empirical Results} \label{sec: 5.1}

\begin{figure*}[t]
  \centering
  \begin{subfigure}[b]{\textwidth}
    \centering
    \begin{subfigure}[b]{0.18\textwidth}
      \includegraphics[width=\textwidth]{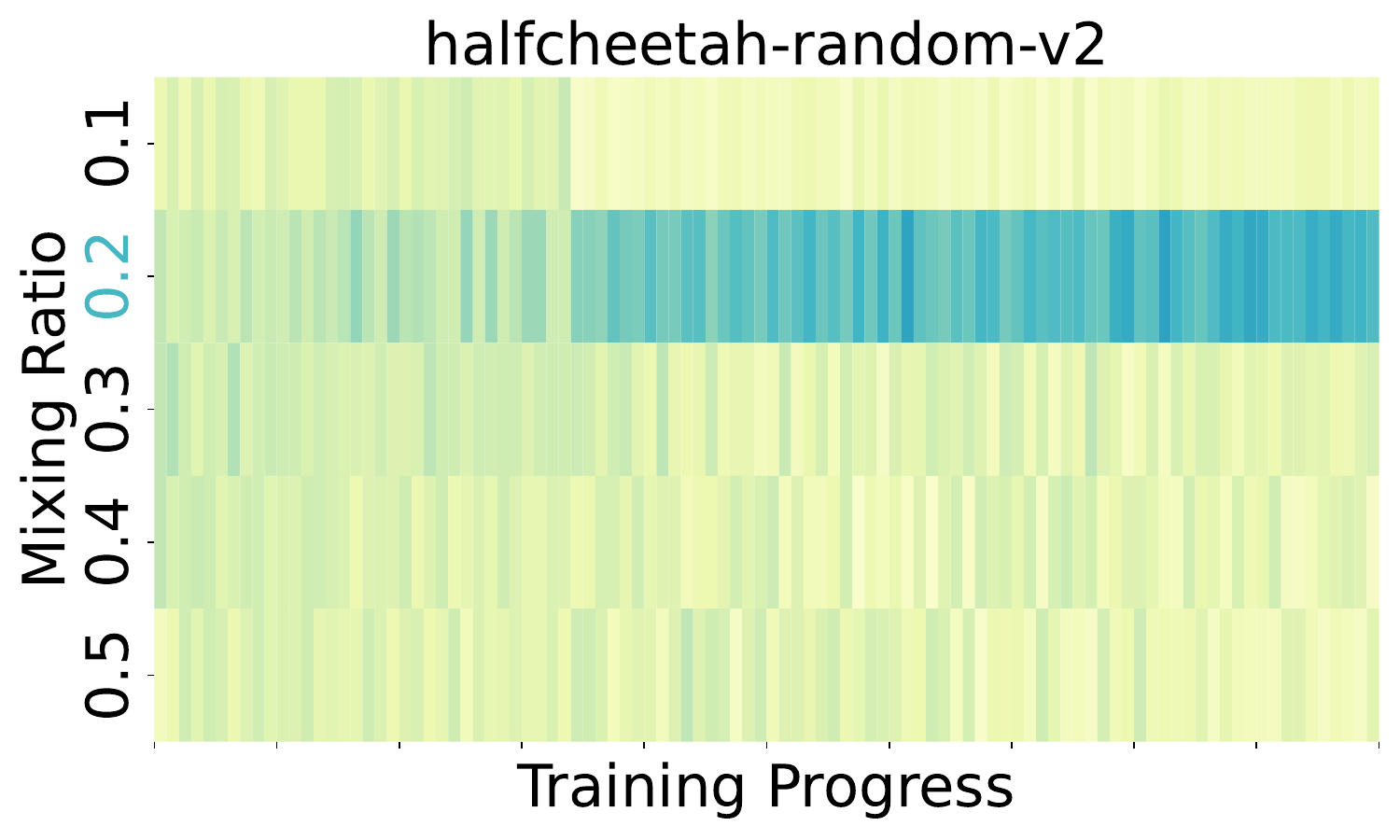}
    \end{subfigure}
    \begin{subfigure}[b]{0.18\textwidth}
      \includegraphics[width=\textwidth]{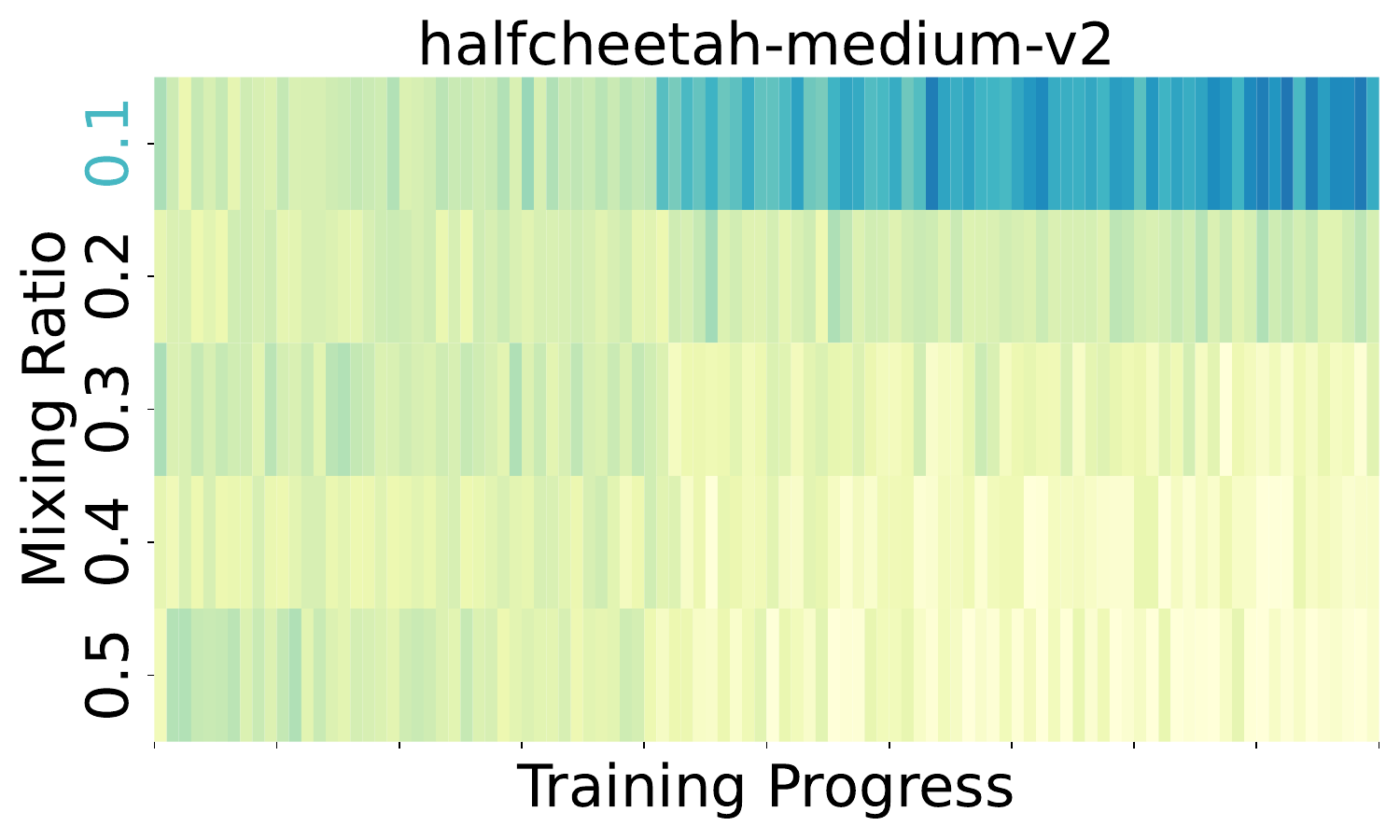}
    \end{subfigure}
    \begin{subfigure}[b]{0.18\textwidth}
      \includegraphics[width=\textwidth]{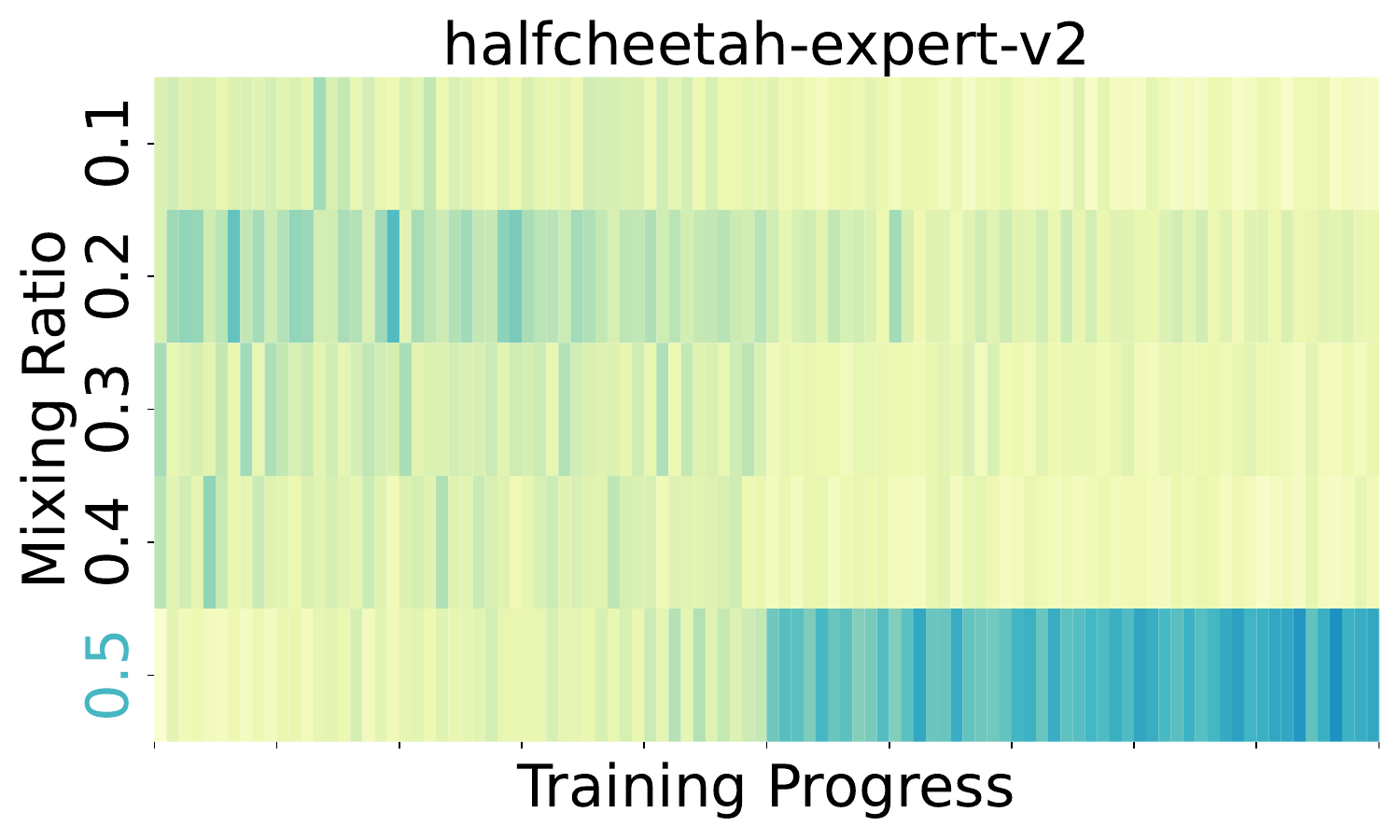}
    \end{subfigure}
    \begin{subfigure}[b]{0.18\textwidth}
      \includegraphics[width=\textwidth]{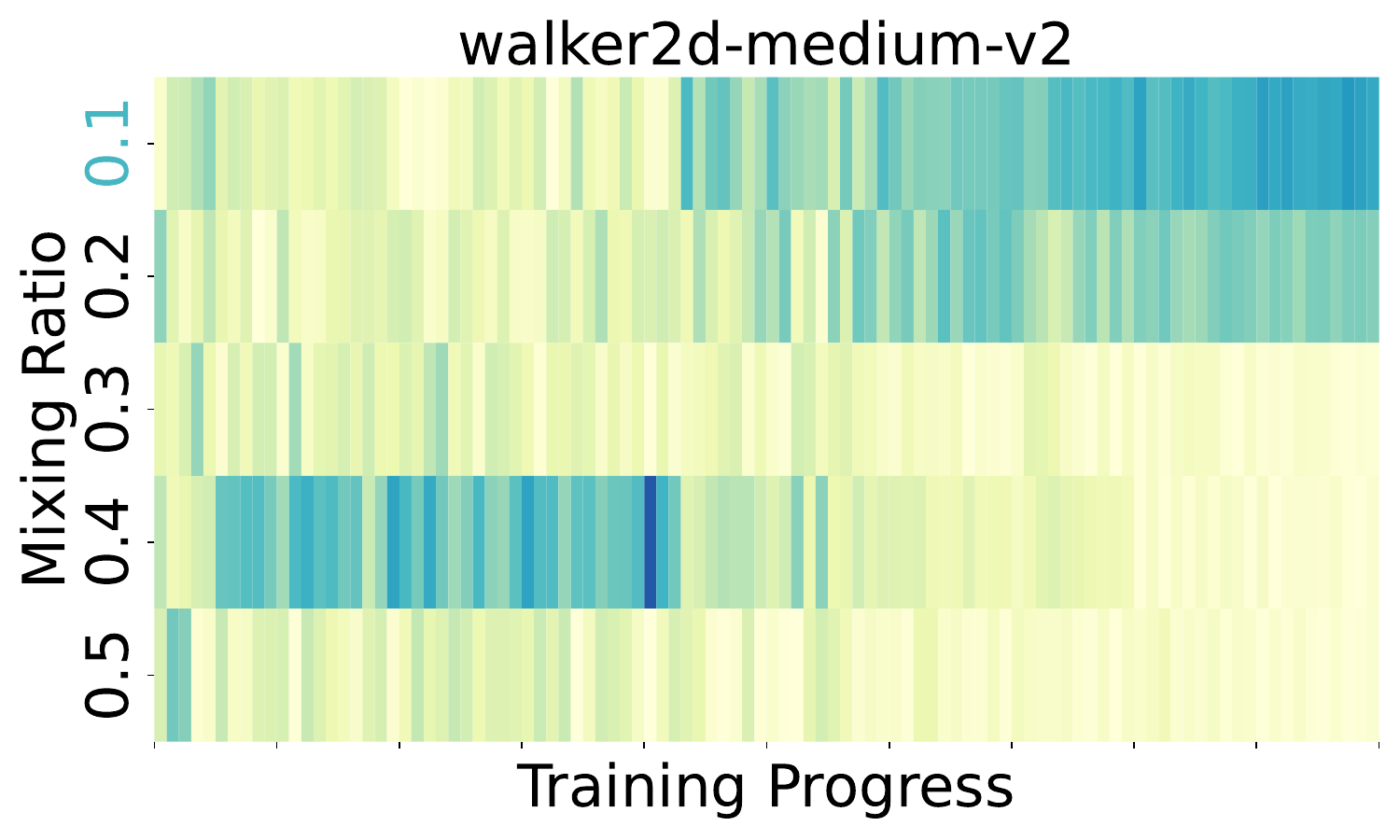}
    \end{subfigure}
    \begin{subfigure}[b]{0.18\textwidth}
      \includegraphics[width=\textwidth]{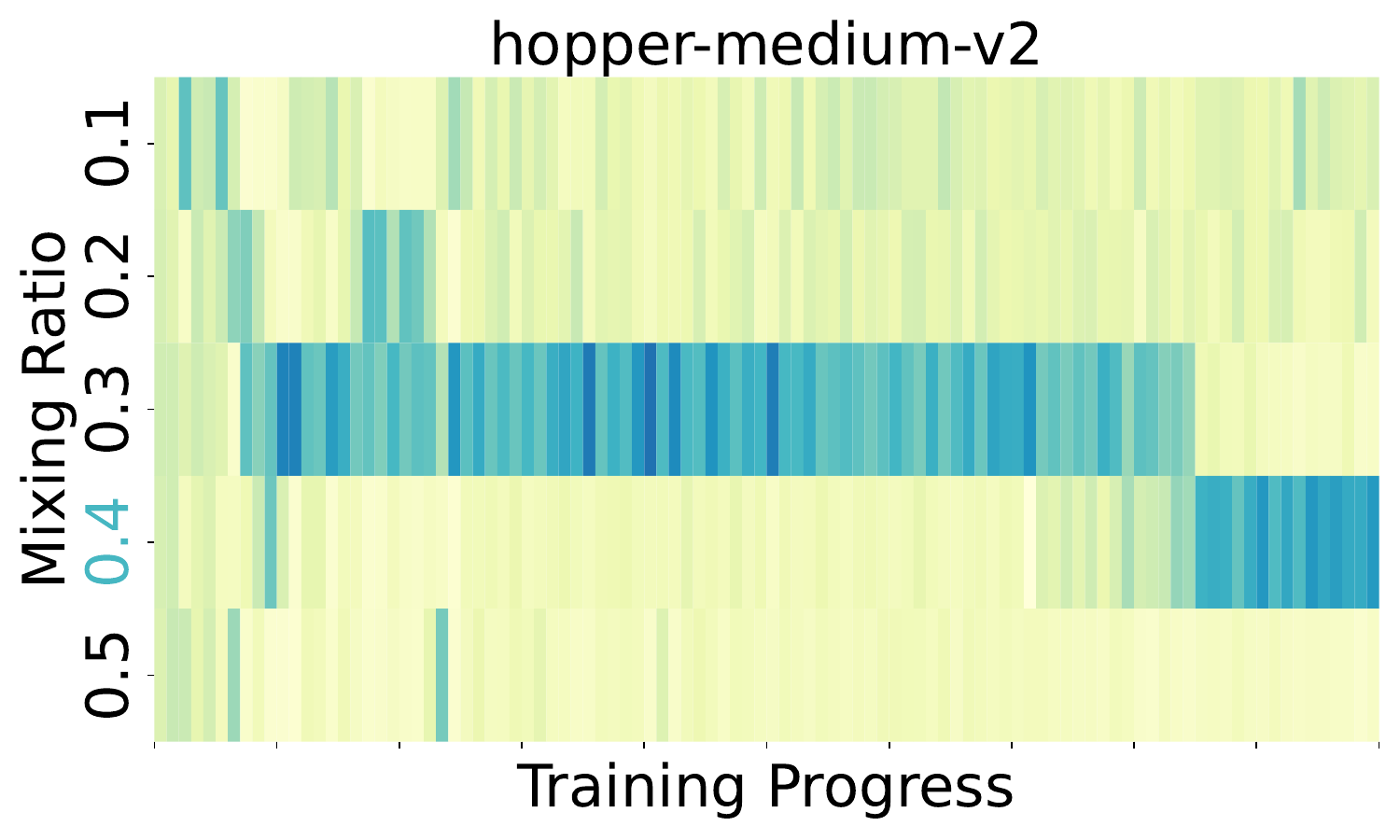}
    \end{subfigure}
    \begin{subfigure}[b]{0.035\textwidth}
      \includegraphics[width=\textwidth]{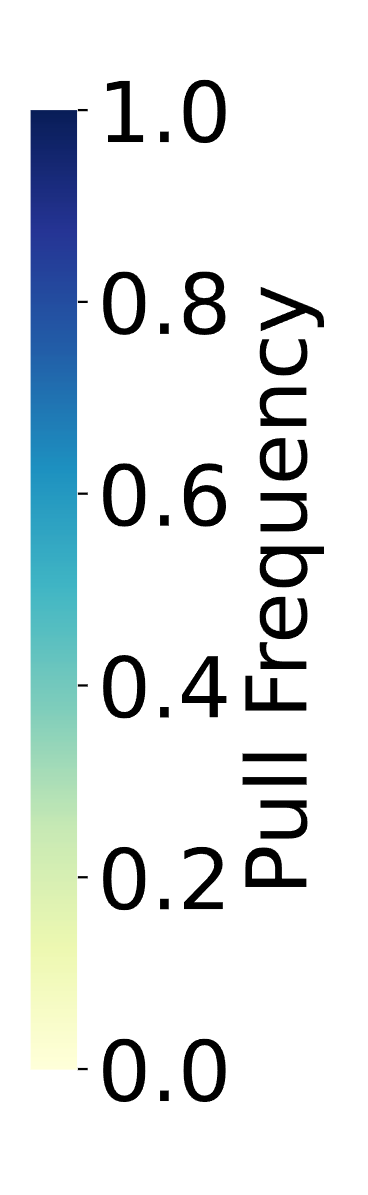}
    \end{subfigure}
  \end{subfigure}
  \caption{Heatmap for mixing ratio selection in ROAD. In some cases, ROAD could learn to converge to the mixing ratio that performs best when fixed (left three), while in some other cases it could adjust its mixing ratio pattern more adaptively (right two).
  }
  \label{fig: heatmap}
  \vspace{0.5cm}
\end{figure*}

\paragraph{ROAD shows solid performance and adaptivity across various environments.}
Table \ref{tab:performance_metrics2} provides an analysis of ROAD's performance using normalized metrics with IQL-based offline phase as an example. \textbf{\texttt{Fixed}} mixing ratios have a performance heavily dependent on the chosen ratio. While some environments are less sensitive, others exhibit pronounced differences. Heuristic-based strategies like \textbf{\texttt{Decreasing}} and \textbf{\texttt{BR}} offer improved robustness to data quality variations but lack consistent superiority. In contrast, ROAD presents consistently strong or comparable performance across all tested environments, emerging as the most stable and adaptive method evaluated. The comparison against \textbf{\texttt{Uniform}} arm selection demonstrates the efficacy of our proposed surrogate $R_q$.

\paragraph{ROAD shows adaptivity across various offline data qualities and RL algorithms.}
In our benchmarking environments, ROAD showcases its ability to handle diverse offline data qualities. The performances of baselines exhibit varying performance depending on dataset quality while struggling to adapt to changing data quality. For example, when an initially inferior offline policy undergoes training and outperforms the offline dataset, it is generally intuitive to switch the data mixing strategy. By contrast, ROAD consistently matches or exceeds baseline performance across different dataset qualities. This adaptability extends to various RL algorithms, including PEX, Cal-QL and CQL, whose results are shown in Appendix \ref{Appendix: result of Cal-QL and CQL}. ROAD also effectively integrates online RL algorithms with offline data, even without offline pretraining, as demonstrated in Appendix \ref{appendix: result of RLPD} with RLPD \cite{rlpd}. These results highlight ROAD’s flexibility and robustness in diverse settings, adapting seamlessly to varying algorithms and offline data conditions.

\begin{table*}[t]
\scriptsize
\centering
\caption{Ablation study on the offline and online components of relative policy quality in ROAD.}
\begin{tabular}{@{}lccccccccc@{}}
\toprule
Hyperparameter choice of $\kappa$                      & $0$ ($R_q=\Delta_\text{off}$) & $0.2$      & $0.5$     & $1.0$       & $2.0$        & $5.0$       & $\infty$ ($R_q=-\Delta_\text{on}$) \\ \midrule
halfcheetah-medium        & $72.25 \pm 2.70$                     & $75.05 \pm 1.10$  & $77.09 \pm 0.91$ & $74.57 \pm 1.95$ & $76.33 \pm 1.56$  & $74.33 \pm 1.48$ & $68.31 \pm 1.86$                            \\
halfcheetah-medium-replay & $55.08 \pm 1.29$                     & $57.25 \pm 2.71$  & $53.51 \pm 1.68$ & $55.82 \pm 1.07$ & $55.16 \pm 0.82$  & $54.74 \pm 1.83$ & $47.98 \pm 3.41$                            \\
hopper-medium             & $82.88 \pm 15.84$                    & $85.60 \pm 16.13$ & $91.58 \pm 4.46$ & $99.23 \pm 1.05$ & $97.11 \pm 4.57$  & $94.15 \pm 4.75$ & $73.96 \pm 8.17$                            \\
hopper-medium-replay      & $52.78 \pm 14.83$                    & $98.28 \pm 7.81$  & $96.11 \pm 6.49$ & $99.14 \pm 9.97$ & $85.00 \pm 14.93$ & $93.93 \pm 6.29$ & $93.39 \pm 3.73$                            \\
\bottomrule
\end{tabular}
\label{table: ablation study}
% \vspace{0.5cm}
\end{table*}

\subsection{Understanding ROAD}

\paragraph{The dynamics of $R_q$.}
Our surrogate objective $R_q$ leverages in-training statistics to capture the training dynamics and guide data mixing decisions accordingly. To intuitively illustrate the physical interpretation of our bi-level optimization theory, we construct a simple MDP and a corresponding offline dataset, and plot the curves of $\Delta_{\text{off}}$ and $\Delta_{\text{on}}$ throughout the online fine-tuning phase. As is depicted by Figure \ref{fig:examplar-mdp}, $\Delta_{\text{off}}$ consistently acts as a proxy for policy improvement. Conversely, early fine-tuning is safely constrained by minimizing $\Delta_{\text{on}}$. $\Delta_{\text{on}}$ later decays in scale, leaving it as a noise term in $R_q$ and thus preventing it from suppressing true policy improvement. More details are provided in Appendix \ref{Appendix: simple MDP}.
\begin{figure}[htbp]
    \centering
    \includegraphics[width=0.8\linewidth]{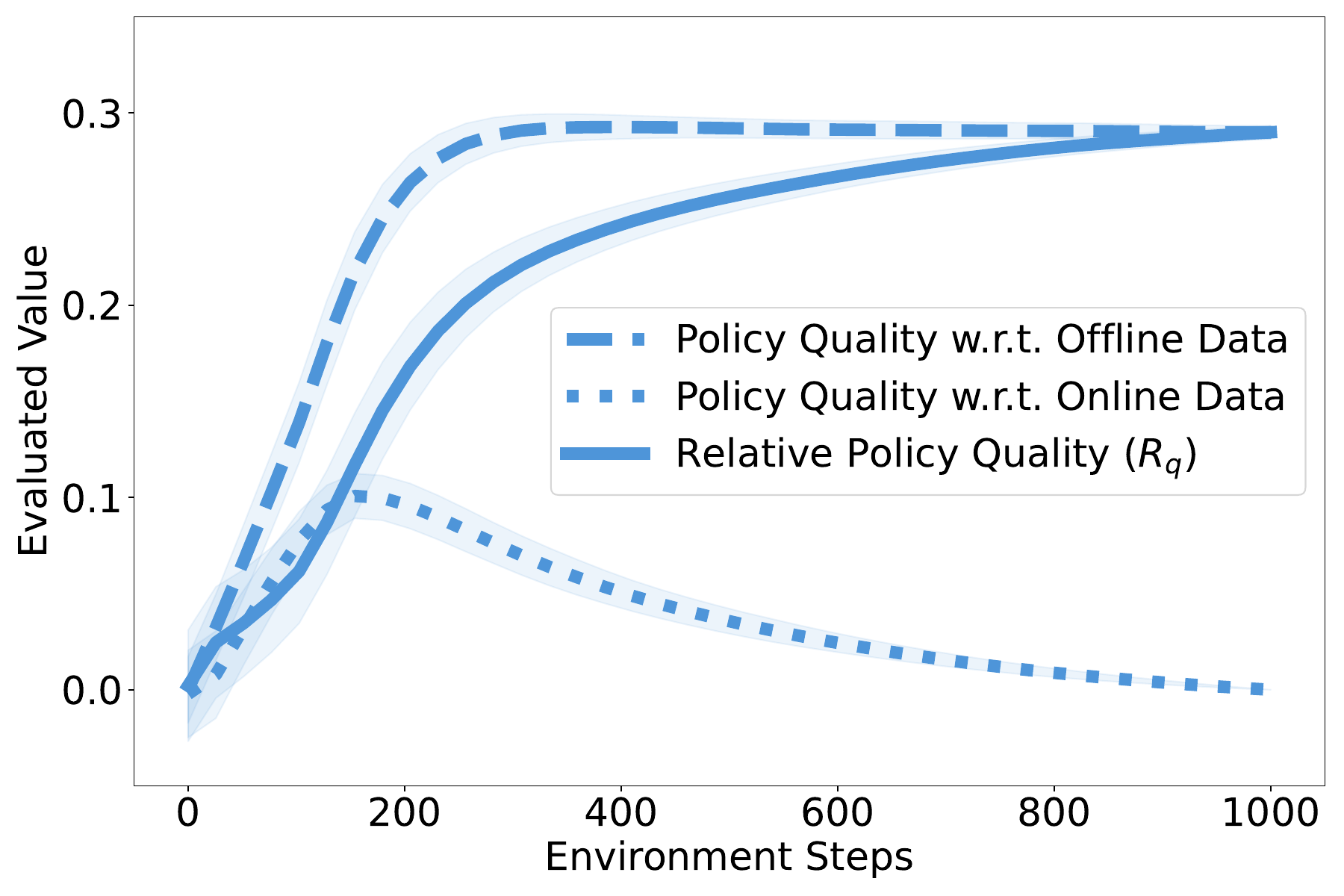}
    \caption{The $R_q$ curve in online fine-tuning.}
    \label{fig:examplar-mdp}
\end{figure}

\paragraph{The behavior of ROAD.}
To analyze how ROAD adapts mixing ratios of offline and online data across environments, we visualized its decisions during training using heatmaps. These heatmaps, shown in Figure \ref{fig: heatmap}, capture the evolution of mixing ratios for different environments and offline datasets. Specifically, different environments witness different type of data mixing strategies. In some cases, ROAD converges to a fixed ratio, while in other cases, ROAD adapts its strategy to the altered training dynamics. By exploring diverse ratios early in training, ROAD enables policies to become more robust to varying state-action distributions, leading to improved learning outcomes and more effective integration of data sources.

\subsection{Ablation and Parameter Sensitivity Study}

\paragraph{The hyperparameter $\kappa$.} This hyperparameter effectively balances aggressive policy improvement and conservative offline data utilization. Two extremes are $\kappa=0$ and $\kappa \to \infty$ corresponding to ablating the contributions of the offline and online components of relative policy quality. Results in Table~\ref{table: ablation study} demonstrate that the full ROAD framework with both offline and online components consistently outperforms settings that rely solely on one component. However, ROAD demonstrates low sensitivity to $\kappa$ when $\kappa$ falls within a moderate range. In these cases, ROAD works as expected, functioning as a stabilizer initially (dominated by $\Delta_\text{on}$) and optimizes asymptotic performance later on (dominated by $\Delta_\text{off}$).

\paragraph{The sliding-window size $\tau$ and exploration parameter $c$.}
Two critical parameters in balancing the exploration and exploitation of the ROAD framework are the exploration parameter $c$ and the sliding-window size $\tau$. To explore the effects of these parameters, we conducted experiments on the Halfcheetah environment with various combinations of $c$ and $\tau$, as detailed in Table \ref{tab: parameter c}. The results demonstrate that ROAD's performance remains robust across different parameter settings, highlighting its resilience to parameter variations. Typically, setting $\tau = 1000$ better captures the slowly-varying training dynamics in most environments, leading to better asymptotic performance. This indicates ROAD's broad adaptability to different environments and offline datasets.

\begin{table}
\scriptsize
\centering
\caption{ROAD performance across $c$ and $\tau$ in Halfcheetah (mean).}
\label{tab: parameter c}
\begin{tabular}{p{1.2cm}p{1.2cm}p{1.2cm}p{1.2cm}}
\toprule
& \multicolumn{3}{c}{$\tau$} \\
\cmidrule{2-4}
$c$ & $500$ & $1000$ & $\tau = t$ \\
\midrule
$0.5$ & 72.82 & 71.52 & 71.49 \\
$1.0$ & 73.19 & 74.08 & 72.33 \\
$2.0$ & 72.80 & 74.34 & 73.05 \\
$4.0$ & 71.72 & 73.90 & 74.16 \\
\bottomrule
\end{tabular}
\end{table}

\section{Conclusions, Limitations and Future Work} \label{sec 6}

In this work, we provide a novel perspective on how data mixing affects offline-to-online RL fine-tuning. Under this perspective, we identify the objective misalignment problem, where minimizing the Bellman error does not translate to maximized policy performance. To address this misalignment, we propose \textbf{R}einforcement Learning with \textbf{O}ptimized \textbf{A}daptive \textbf{D}ata-mixing, a novel, plug-and-play framework that formulates data mixing as a bi-level optimization problem. By treating data mixing strategy as a meta-decision variable, ROAD dynamically aligns the inner-level Q-learning objective with the goal of optimizing policy performance.

Our theoretical analysis of the optimization problem yields a gradient-based solution. Central to bridging the theory and practice is the derivation of a theoretically grounded surrogate objective $R_q$ which approximates the unbiased estimator of the outer-level objective. This surrogate encourages policy improvement in trusted regions (via $\Delta_\text{off}$) and suppresses value overestimation (via $\Delta_\text{on}$). Extensive empirical evaluations demonstrate that ROAD consistently outperforms existing methods dependent on static and heuristic-based data replay strategies. ROAD exhibits robustness across various backbone algorithms and diverse offline data qualities.

Despite the theoretical rigor and promising empirical results, ROAD has limitations that merit discussion. The primary limitation lies in the assumption of linearized mixing strategies and further discretizations of the search space. These assumptions ensure computational feasibility but are essentially projecting our theoretical solution in the infinite-dimensional function space $\mathcal{G}$ to a low-dimensional subspace, compromising the rich expressivity of the function space. Another limitation stems from Assumption \ref{assumption:offline-uncertainty}, indicating reliance on the offline dataset. In scenarios with extremely sparse or narrow offline datasets where the Q-value estimation is not sufficiently anchored, the reliability of $R_q$ might degrade.

The promising directions primarily correspond to the limitation of restricted strategy search spaces. One natural extension to ROAD is to incorporate sample-level data selection (as suggested by the $w_g$ term in Equation \eqref{eq:outer-level-grad}), effectively introducing larger search spaces of mixing strategies. Another future work would be exploring methods to search for the optimal mixing strategy in a continuous space, thus allowing for more delicate control over the training dynamics. Apart from these methodological extensions, the application of this framework to real-world robotic tasks, which emphasizes the algorithm's ability to prevent unlearning while ensuring asymptotic performance, is also of interest.

\section*{Acknowledgments}

This work was supported by Ant Group Research Fund.

\bibliographystyle{named}
\bibliography{citation}

% \include{sections/0-abstract}
% \include{sections/1-introduction}
% \include{sections/2-related_work}
% \include{sections/3-theory}
% \include{sections/4-method}
% \include{sections/5-experiment}
% \include{sections/6-conclusion}

% \section*{Acknowledgments}
% This work was supported by Ant Group Research Fund.

% \bibliographystyle{ijcai-2026/named}
% \bibliography{citation}

\newpage
\onecolumn
\appendix

\section{Solution to the Bi-Level Optimization}
\label{sec:appendix-theoretical-solution}

As is demonstrated in Equation \eqref{eq:bi-level-objective}, we solve a bi-level optimization problem that selects an optimal data mixing strategy. Following the bi-level optimization theory \cite{bi-level-opt}, we compute the gradient of the outer-level objective for a theoretical gradient ascent algorithm. Although we are solving an optimization problem within each FQI iteration, all the iterations share the same formulation with only differences in the available data sources. Therefore, we drop the iteration superscript $k$ and denote the outer-level and inner-level objectives as $\mathcal{J}_{\text{out}}$ and $\mathcal{J}_{\text{in}}$ here.

\paragraph{The dependency chain of variables.} The bi-level optimization problem is established on a rigorous causal chain as shown below. Specifically, the data mixing strategy $g$ determines the sampling distribution $d^{\text{sample}}_g$, shaping the loss landscape of $\mathcal{J}_{\text{in}}$. Optimizing this inner-level objective leads to the optimal Q-function $f^*(g)$ and the corresponding policy $\pi(f^*(g))$. Then we evaluate this policy in the environment and obtain an empirical approximation of the outer-level objective (the expected performance of the online policy).
\begin{align*}
    g \xrightarrow{\text{inner-level optimization}} f^*(g) \xrightarrow{\text{soft-greedy}} \pi(f^*(g)) \xrightarrow{\text{exploration}} \mathcal{D}^{\text{online}} \xrightarrow{\text{evaluation}} \mathcal{J}_{\text{out}},
\end{align*}
In the following theoretical derivation, we explicitly takes the dependency chain of variables into account and compute the outer-level gradient in the functional space. For clarify of the dependency structure, we rewrite the objective $\mathcal{J}_{\text{out}}(g)$ as $\mathcal{J}_{\text{out}}(g, f^*(g), \pi(f^*(g)))$.
% we mainly focus on how $g$ and $f^*(g)$ influence the outer-level objective, while abstracting away the intermediate effects of the online policy. Therefore, for clarify of the dependency structure, we rewrite the objective $\mathcal{J}_{\text{out}}(g)$ as $\mathcal{J}_{\text{out}}(g, f^*(g))$.

\paragraph{The Bi-Level Optimization Problem Formulation}

In this theoretical derivation, we follow the notations from RL theory works such as Cal-QL \cite{cal-ql}. We assume the Q-function space $\mathcal{F}$ is the Hilbert space $L^2(\mathcal{S}\times\mathcal{A})$ defined on the state-action space $\mathcal{S}\times\mathcal{A}$, within which the inner product is defined by $\langle f,h \rangle_{\mathcal{F}} =\int_{\mathcal{S}\times\mathcal{A}} f(s,a) h(s,a) \mathrm{d}s\mathrm{d}a$. We assume that the state and action space $\mathcal{S}$ and $\mathcal{A}$ are both compact sets (so that the integrals on these sets exists). In our problem, the data mixing strategy $g$ maps the input datasets to a sampling distribution $d_g \in \Delta(\mathcal{S}\times\mathcal{A})$, thereby defining the inner-level optima $f^*(g)$ through the inner-level objective $\mathcal{J}_{\text{in}}$.
\begin{align}
    \mathcal{J}_{\text{in}}(g,f) = \mathbb{E}_{(s,a) \sim d^{\text{sample}}_g} \left[ \left( \mathcal{T}f^{k}(s,a) - f(s,a) \right)^2 \right] = \int_{\mathcal{S}\times\mathcal{A}} d^{\text{sample}}_g(s,a) \left( \mathcal{T}f^{k}(s,a) - f(s,a) \right)^2 \mathrm{d}s\mathrm{d}a
    \nonumber
\end{align}
With this inner-level optimization, we assume a differentiable mapping $\Psi:\mathcal{F}\to \Pi$ that induces $\pi=\Psi(f)$ from a Q-function $f$. One typical example could be the softmax mapping where $\pi(a|s) \propto \exp\left( \beta f(s,a) \right)$. Finally, the outer-level objective is defined by $\mathcal{J}_{\text{out}}$, where $\mu\in\Delta(\mathcal{S})$ is the initial state distribution as defined in Section \ref{sec:theory-preliminary}.
\begin{align}
    \mathcal{J}_{\text{out}}(g,f^*(g),\pi(f^*(g)))
    % = J(\pi(f^*(g)))
    % = \mathbb{E}_{s\sim \mu, \tau \sim \pi(f^*(g))} \left[ \sum_{t} \gamma^t r_t \vert s_0=s \right]
    = \mathbb{E}_{(s,a) \sim d^{\pi(f^*(g))}} \left[ f^*(s,a) \right]
    = \int_{\mathcal{S}\times\mathcal{A}} d^{\pi(f^*(g))}(s,a) f^*(s,a) \mathrm{d}s\mathrm{d}a
    \nonumber
\end{align}

\subsection{Computing the Outer-Level Gradient}

To compute the outer-level gradient $\nabla_g \mathcal{J}_{\text{out}}$, we apply the chain rule and decompose it into several parts. The $D_g f^*$ term is computed via the implicit function theorem, and $\nabla_\pi \mathcal{J}_\text{out}$ is essentially the policy gradient in reinforcement learning.
\begin{align}
    \nabla_g \mathcal{J}_\text{out} = \left( D_g f^* \right)^\dagger \nabla_f \mathcal{J}_\text{out}
    \label{eq:comp-decompose-outer-grad}
\end{align}

\paragraph{The implicit inner-level gradient.} The inner-level optimality condition is that $\nabla_f \mathcal{J}_{\text{in}}(g, f^*(g))=0$ (notably, in $\mathcal{J}_\text{in}$, the two variables $f$ and $g$ are independent). Let $\Phi(g,f)=\nabla_f \mathcal{J}_\text{in}(g,f)$ where $f=f^*(g)$, thus introducing the dependency of $f$ on $g$. We take the derivative with respect to the data mixing strategy $g$ and obtain Equation \eqref{eq:relating_g_f^*}.
% \begin{align}
%     & \nabla_{g} \left( \partial_f \mathcal{J}_{\text{in}}(g, f^*(g)) \right) = \partial_{g,f} \mathcal{J}_{\text{in}} + \partial_f^2 \mathcal{J}_{\text{in}} \cdot \partial_g f^*(g) = 0 \nonumber\\
%     \Rightarrow & \partial_g f^*(g) = - \left( \partial_f^2 \mathcal{J}_{\text{in}} \right)^{-1} \partial_{g,f} \mathcal{J}_{\text{in}}
%     \label{eq:relating_g_f^*}
% \end{align}
\begin{align}
    D_g \Phi(g,f^*(g))=0
    & \Rightarrow
    \partial_g \Phi + \partial_f \Phi \circ D_g f^* = 0
    \Rightarrow
    D^2_{gf}\mathcal{J}_\text{in} + D^2_{f} \mathcal{J}_\text{in} \circ D_g f^* = 0
    \nonumber\\
    & \Rightarrow
    D_g f^* = - \left( D^2_{f} \mathcal{J}_\text{in} \right)^{-1} \cdot D^2_{gf}\mathcal{J}_\text{in}.
    \label{eq:relating_g_f^*}
\end{align}
With the explicit form of inner-level objective $\mathcal{J}_{\text{in}}$, we can translate these abstract operators into practical forms with clear semantics. Let $\delta_f^{k}(s,a) = f(s,a) - \mathcal{T}f^{k}(s,a)$ be the temporal difference error, $h\in\mathcal{F}$ be a perturbation, then we have:
\begin{align}
    & D_f \mathcal{J}_{\text{in}}(g,f)[h]
    = \int_{\mathcal{S}\times\mathcal{A}} d^{\text{sample}}_g(s,a) \frac{\partial}{\partial f} \left( \delta_f^k(s,a) \right)^2 h(s,a)\mathrm{d}s\mathrm{d}a
    % = 2\int_{\mathcal{S}\times\mathcal{A}} d^{\text{sample}}_g(s,a) \delta_f^k(s,a) h(s,a)\mathrm{d}s\mathrm{d}a
    = \left\langle 2 d_g^{\text{sample}} \delta_f^k , h \right\rangle_{\mathcal{F}}
    \nonumber\\
    \Rightarrow &
    (\nabla_f \mathcal{J}_{\text{in}})(s,a) = 2 d_g^{\text{sample}}(s,a) \delta_f^k (s,a)
    % \nonumber\\
    \Rightarrow
    \begin{cases}
        D^2_{f} \mathcal{J}_\text{in}(g,f)[h] = D_f \left( \nabla_f \mathcal{J}_{\text{in}} \right)[h] = \left\langle 2 d_g^{\text{sample}} , h \right\rangle_{\mathcal{F}}
        \\
        D^2_{gf}\mathcal{J}_\text{in}(g,f)[h] = D_g \left( \nabla_f \mathcal{J}_{\text{in}} \right)[h] = \left\langle 2 \delta_f^k \left( \nabla_g d_g^{\text{sample}} \right), h \right\rangle_{\mathcal{F}}
    \end{cases},
    \nonumber
\end{align}
where we utilize the fact that the Hessian operator for the weighted least squares objective $\mathcal{J}_{\text{in}}$ is self-adjoint. Substituting this into Equation \eqref{eq:comp-decompose-outer-grad}, we have
\begin{align}
    \nabla_g \mathcal{J}_\text{out} = -\int_{\mathcal{S}\times\mathcal{A}} \frac{\nabla_f \mathcal{J}_{\text{out}}(s,a)}{d_g^{\text{sample}}(s,a)} \left( \nabla_g d_g^{\text{sample}} \right) (s,a) \delta_f^k(s,a) \mathrm{d}s\mathrm{d}a.
    \label{eq:substitute-inner-level-grad}
\end{align}

\paragraph{Connecting $f$ and $\mathcal{J}_{\text{out}}$.}
The remaining problem is to take the derivative of $\mathcal{J}_{\text{out}}(\pi,f)=\mathbb{E}_{(s,a) \sim d^{\pi}} \left[ f(s,a) \right]$ with respect to $f$. Since the outer-level objective is dependent on both $f$ and $\pi$, the derivative is decomposed as Equation \eqref{eq:comp-outer-grad-on-f}.
\begin{align}
    D_f \mathcal{J}_{\text{out}} = \partial_f \mathcal{J}_{\text{out}} + (D_f \pi)^\dagger \partial_\pi \mathcal{J}_\text{out}
    \label{eq:comp-outer-grad-on-f}
\end{align}
The remaining three terms are straightforward enough to be directly interpreted. The resulting gradient is
\begin{align}
    \left\langle \nabla_f \mathcal{J}_\text{out}, h \right\rangle_{\mathcal{F}} = \partial_f \mathcal{J}_{\text{out}} [h] + (D_f \pi)^\dagger \partial_\pi \mathcal{J}_\text{out} [h]
    = \left\langle d^{\pi}, h \right\rangle_\mathcal{F} + \left\langle (D_f \pi)^\dagger d^{\pi} f, h\right\rangle_\mathcal{F}.
\end{align}
Here the policy gradient term uses the commonly adopted assumption that $f\approx Q^{\pi}$. This further leads to a compact form of the complete outer-level gradient as is shown below:
\begin{align}
    \nabla_g \mathcal{J}_\text{out}
    =& - \int_{\mathcal{S}\times\mathcal{A}} w_g(s,a) \partial_{gf}^2\mathcal{J}_{\text{in}} (s,a) \mathrm{d}s\mathrm{d}a
    \nonumber
    \\
    =& - \int_{\mathcal{S}\times\mathcal{A}} \frac{d^{\pi}(s,a)\left( 1+ (D_f \pi)^\dagger [f] (s,a)\right)}{d_g^{\text{sample}}(s,a)} \left( \nabla_g d_g^{\text{sample}} \right) (s,a) \delta_f^k(s,a) \mathrm{d}s\mathrm{d}a,
\end{align}
where the density ratio term $w_g$ in Equation \eqref{eq:outer-level-grad} follows the definition above.

\subsection{Interpreting the Theoretical Gradient}
\label{sec:appendix-interpret-theoretical-gradient}

To make this abstract gradient term more intuitive, we assume the policy $\pi(f^*)$ is a softmax policy with respect to $f^*$ with an inverse temperature $\beta$, and the data mixing strategy is a linear combination of online and offline distributions, $d_m(s,a) = md^{\text{off}}(s,a)+(1-m)d^{\text{on}}(s,a)$, as is discussed in Section \ref{sec:road}. The first assumption holds for most of the practical online RL algorithms like max-entropy RL \cite{sac} and trust region methods \cite{trpo}. The second assumption conveys the critical intuition that the data mixing strategy is a tradeoff between utilizing offline and online data.

With these assumptions, we simplify the gradient estimation as
\begin{align}
    \nabla_m \mathcal{J}_\text{out}
    =& - \int_{\mathcal{S}\times\mathcal{A}} \frac{d^{\pi}(s) \pi'(a|s)}{d_m^{\text{sample}}(s,a)} \left( d^{\text{on}} (s,a) - d^{\text{off}} (s,a) \right) \delta_f^k(s,a) \mathrm{d}s\mathrm{d}a,
\end{align}
where $\pi'(a|s) = \pi(a|s) \left( 1+A_f(s,a) \right)$ is a policy induced from the online policy and the advantage function $A_f(s,a)=f(s,a) -\mathbb{E}_{a\sim\pi} \left[ f(s,a)\right]$. This equation implies that, an optimal mixing strategy effectively balances the sensitivity (TD error) of online and offline samples. When the online TD error is large, which indicates instability of online training, then we should add more offline data as an anchor. By contrast, with the training progresses and the bias in offline knowledge is gradually recognized, we should put more emphasis on the online data. This insight aligns with our intuition that offline data dominates the early-stage training while online data benefits long-term asymptotic performance.

\subsection{The Theoretical Algorithm}
\begin{algorithm}[t]
\caption{Bi-level optimization for online RL fine-tuning}
\label{alg:theoretical-bi-level-o2o}
\begin{algorithmic}[1]
    \REQUIRE Value function class $\mathcal{F}$, data mixing function class $\mathcal{G}$, \# total iterations $K$, offline dataset $\mathcal{D}^{\text{offline}}$, offline policy $\pi^{\text{off}}$ and the offline Q-function $f^{\text{off}}$.
    
    \STATE Initialize $f^1(s,a)=f^{\text{off}}(s,a),\forall (s,a)$ and $g \in \mathcal{G}$
    \FOR {$k=1,...,K$}
        \STATE Let $\pi^k$ be soft-greedy w.r.t. $f^k$.
        \STATE Collect online data $\mathcal{D}^k \sim \pi^k$.
        \STATE Solve the bi-level optimization for $g^*$ by setting Equation \eqref{eq:outer-level-grad} to zero.
        \STATE Construct $d^{\text{sample}}_{g^*} = g^* \left( \mathcal{D}^{\text{offline}}, \mathcal{D}^1,...,\mathcal{D}^k \right)$.
        \STATE Implement Q-learning using the constructed distribution $f^{k+1} \leftarrow \arg\min_f \mathcal{J}^{k}_{\text{in}}(g^*,f)$.
    \ENDFOR
    \RETURN $\pi^K$
\end{algorithmic}
\end{algorithm}

An ideal solution to the bi-level optimization problem is demonstrated by Algorithm \ref{alg:theoretical-bi-level-o2o}. Within each FQI iteration, the algorithm determines the best data mixing strategy and carries out Q-learning under the optimal data distribution. Despite the theoretical guarantee of optimal online evaluation performance, this algorithm is largely infeasible in practice.
\begin{itemize}
    \item First, the space of data mixing strategies is a function space with infinite dimension, and there is also no intuitive method to parameterize this function. Therefore, it is almost impossible to directly compute the functional gradient $\nabla_g \mathcal{J}_{\text{out}}$ and implement gradient ascent. In reality, we have to constrain the family of candidate strategies into a smaller subset of $\mathcal{G}$.
    \item The second infeasibility lies in the distribution estimation terms, or equivalently, the theoretical expectation notations. These terms assume zero epistemic uncertainty and thus no estimation bias. However, it is exactly the estimation bias that yields unlearning at early stages, making O2O RL challenging.
\end{itemize}

\section{Building the surrogate objective.}
\label{sec:build-surrogate}

% 如正文中所说，我们假设实际训练中使用empirical dataset而非理想的distribution上的expectation计算出的Q函数为$\hat{f}(s,a)$，这个随机变量的方差为$\sigma^2(s,a)=\text{Var}(\hat{f})(s,a)$，那么实际的对目标的无偏估计为Equation \eqref{eq:practical-objective}。
As is discussed in Section \ref{sec:road-practical}, the major difference between the theoretical objective and a practical RL fine-tuning objective lies in the nature of expectation. Specifically, the theoretical formulation assumes a true expectation while the practical estimation relies on finite datasets. Moreover, in offline-to-online settings, the reliability of these empirical expectations varies significantly across the state-action space. Adhering to uniform confidence in the estimation leads to exploitation of erroneous Q-values and subsequent performance collapse. To mitigate such risk and ensure safe exploration, we model the lower bound for the theoretical objective.

\subsection{An Unbiased Outer-Level Objective Estimator}
\label{sec:build-surrogate-unbiased-estimator}

% \paragraph{Formulation of the practical objective.}
\begin{assumption}
    The Q-function estimated from the empirical dataset is $\hat{f}(s,a)=f(s,a)+\epsilon(s,a)$. Here the zero-mean (i.e., for any state-action pair $(s,a)$, $\mathbb{E}_{\epsilon} \left[ \epsilon(s,a) \right]=0$) error term $\epsilon(s,a)$ is defined as a random field on $\mathcal{S}\times\mathcal{A}$, with its correlation depicted by a continuous covariance kernel $K((s,a),(s',a'))=\text{Cov}(\epsilon(s,a),\epsilon(s',a'))$. Meanwhile, the error term is independent from the estimation target itself, i.e., $f$ and $\epsilon$ are independent from each other.
    \label{assumption:estimation-error}
\end{assumption}
% Then the variance of the empirical estimation is given by $\sigma^2(s,a)=\text{Var}(\hat{f})(s,a)=K((s,a),(s,a))$.

Under such an assumption we can model the overestimation bias by taking the second-order Taylor expansion (the bias as a function of $\hat{f}$ has a universal optimum at $\hat{f}=f$, so the first-order derivative is zero).
\begin{definition}
    The overestimation bias is defined as the discrepancy between the perceived value (computed via an inaccurate value function $\hat{f}(s,a)$) and the true value.
    \begin{align}
        \text{Bias}_{\hat{f}}
        = \mathbb{E}_{a\sim d^{\pi(\hat{f})}} \left[ \hat{f}(s,a) - f(s,a) \right]
        = \int_{\mathcal{A}} \pi_{\hat{f}} (a|s) \cdot \epsilon(s,a) \mathrm{d}a
        \label{eq:definition-overestimation-bias}
    \end{align}
\end{definition}

\begin{theorem}
    Under Assumption \ref{assumption:estimation-error}, the overestimation bias at each state $s$ is computed by Equation \eqref{eq:comp-estimation-overestimation-bias}, where $\Vert R_3 \Vert = o\left( \Vert \epsilon(s) \Vert^3 \right)$ is the third-order remainder. Here the terms involving the state $s$ are ignored for symbolic simplicity, e.g., $K(a,a')$ is an abbreviation for $K((s,a),(s,a'))$.
    \begin{align}
        \mathbb{E}_{\epsilon} \left[ \text{Bias}_{\hat{f}}(s) \right]
        = \iint_{\mathcal{A}\times\mathcal{A}} \frac{\delta \pi (a)}{\delta \hat{f}(a')} K(a,a') \mathrm{d}a \mathrm{d}a' + R_3
        \label{eq:comp-estimation-overestimation-bias}
    \end{align}
\end{theorem}

\begin{proof}
We first model the overestimation bias by Equation \eqref{eq:comp-overestimation-bias}. To approximate this overestimation, we take the second-order Taylor expansion (the bias as a function of $\hat{f}$ has a universal optimum at $\hat{f}=f$, so the first-order derivative is zero). To make the dependency clear, we define the bias function by Equation \eqref{eq:comp-overestimation-bias}, where $f(s,a)$ is considered a constant function, and the function variable $\pi_{\hat{f}}$ depends on $\hat{f}$. Note that in this formulation, we discard the integral on $\mathcal{S}$ and consider it as a function on actions $a\in\mathcal{A}$ because $\pi(\cdot|s)$ is dependent only by the Q-values of the same state $f(s,\cdot)$.
\begin{align}
    \text{Bias}_{\hat{f}}
    % = \mathcal{J}_{\text{out}}(\pi(\hat{f})) - \mathcal{J}_{\text{out}}(\pi(f))
    = \mathbb{E}_{a\sim d^{\pi(\hat{f})}} \left[ \hat{f}(s,a) - f(s,a) \right]
    % = \int_{\mathcal{A}} \pi_{\hat{f}} (a|s) \cdot \left( \hat{f}(s,a) - f(s,a)\right) \mathrm{d}a
    = \int_{\mathcal{A}} \pi_{\hat{f}} (a|s) \cdot \epsilon(s,a) \mathrm{d}a
    % \nonumber\\
    = \frac{1}{2} D_{f}^2 \mathcal{J}_{\text{out}} \left[ \epsilon, \epsilon \right] + R_3.
    \label{eq:comp-overestimation-bias}
\end{align}
The first-order and the second-order derivatives are as follows ($h_1$ and $h_2$ are arbitrary functions in $\mathcal{F}$):
\begin{align}
\begin{cases}
    D_{\hat{f}} \text{Bias}[h_1]
    = \partial_{\hat{f}} \text{Bias}[h_1] + \left( D_{\hat{f}} \pi_{\hat{f}}\right)^\dagger \partial_\pi \text{Bias}[h_1]
    = \int_{\mathcal{A}} \pi_{\hat{f}}(a)h_1(a) \mathrm{d}a + \int_{\mathcal{A}} \left( D_{\hat{f}} \pi_{\hat{f}}\right)^\dagger [\epsilon] (a) h_1(a) \mathrm{d}a
    \nonumber\\
    D_{\hat{f}}^2 \text{Bias} [h_1,h_2]
    = D_{\hat{f}} \left( D_{\hat{f}} \text{Bias}[h_1] \right) [h_2]
    = \int_{\mathcal{A}} \left( D_{\hat{f}} \pi_{\hat{f}}\right)^\dagger [h_1](a) h_2(a) \mathrm{d}a + \int_{\mathcal{A}} \left( D_{\hat{f}} \pi_{\hat{f}}\right)^\dagger [h_2](a) h_1(a) \mathrm{d}a.
\end{cases}
\end{align}
Since the inner-level derivative is usually a multiplication operator, we assume this multiplier to be a function $d^{\pi}(s,a) h^\pi(s,a)$, and thus the expectation of the bias term is estimated by Equation \eqref{eq:comp-expectation-overestimation-bias}.
\begin{align}
    \mathbb{E}_{\epsilon} \left[ \text{Bias}_{\hat{f}} \right]
    & = \mathbb{E}_{\epsilon} \left[ \int_{\mathcal{A}} \left( D_{\hat{f}} \pi_{\hat{f}}\right)^\dagger [\epsilon](a) \epsilon(a) \mathrm{d}a \right] + R_3
    % = \mathbb{E}_{\epsilon} \left[ \int_{\mathcal{A}} h^{\pi}(a) \epsilon^2(a) \mathrm{d}a \right]
    % = \mathbb{E}_{a \sim \pi} \left[ h^{\pi}(a) \sigma^2(a) \right]
    = \mathbb{E}_{\epsilon} \left[ \int_{\mathcal{A}} \left( \int_{\mathcal{A}} \frac{\delta \pi (a)}{\delta \hat{f}(a')}\epsilon(a') \mathrm{d}a' \right) \epsilon(a) \mathrm{d}a \right] + R_3
    \nonumber\\
    &= \iint_{\mathcal{A}\times\mathcal{A}} \frac{\delta \pi (a)}{\delta \hat{f}(a')} K(a,a') \mathrm{d}a \mathrm{d}a' + R_3
    \label{eq:comp-expectation-overestimation-bias}
\end{align}
\end{proof}

\subsection{Approximating The Overestimation Bias via $\Delta_{\text{on}}$}

% We model the overestimation bias $\text{Bias}(\hat{f})$ with the realistic assumption of policy improvement under small learning rates and almost-independent empirical errors. The intuition lies in the fact that, under small learning rates, the real improvement in the policy is bounded by a small value. Therefore, if we perceive a large increment in the policy value, it is likely caused by overestimation. We should constrain this delusion to ensure a stable training process.

The perceived policy improvement with respect to the online replay buffer $\Delta_{\text{on}}$ is decomposed into two terms, the true advantage and the overestimation bias.
\begin{align}
    \mathbb{E}_{\epsilon} \left[ \Delta_{\text{on}}(s) \right]
    &= \mathbb{E}_{\epsilon} \left[ \mathbb{E}_{a\sim\pi} \left[ \hat{f}(s,a)\right] - \mathbb{E}_{a\sim \pi^{k}} \left[ \hat{f}(s,a)\right] \right]
    \nonumber\\
    &= \underbrace{\mathbb{E}_{\epsilon} \left[ \mathbb{E}_{\pi} \left[ f\right] - \mathbb{E}_{\pi^{k}} \left[ f\right] \right]}_{\text{True advantage}} + \underbrace{\mathbb{E}_{\epsilon} \left[ \mathbb{E}_{\pi} \left[ \hat{f}\right] - \mathbb{E}_{\pi} \left[ f \right] \right]}_{\text{Bias}} - \underbrace{\mathbb{E}_{\epsilon} \left[ \mathbb{E}_{\pi^{k}} \left[ \hat{f} \right] - \mathbb{E}_{\pi^{k}} \left[ f\right] \right]}_{\text{Bias}'=0},
    \\
    \text{where }&
    \text{Bias}'
    % \mathbb{E}_{\epsilon} \left[ \mathbb{E}_{\pi^{k}} \left[ \hat{f}(s,a) \right] - \mathbb{E}_{\pi^{k}} \left[ f(s,a)\right] \right]
    = \mathbb{E}_{\epsilon} \left[ \int_{\mathcal{A}} \pi^{k}(a|s) \epsilon(a) \mathrm{d}a \right]
    = \int_{\mathcal{A}} \pi^{k}(a|s) \mathbb{E}_{\epsilon} \left[\epsilon(a) \right] \mathrm{d}a
    =0.
    \nonumber
\end{align}
The task is to compare the overestimation bias term and the true advantage term. We approximate the true advantage term by a first-order Taylor expansion on the online policy $\pi$. (In fact, the former approximation of the bias term is also equivalent to a first-order policy expansion.) This expansion eliminates the stochastic term within the true advantage, and the remaining problem lies in the comparison between the deterministic true advantage term and the stochastic overestimation bias.
\begin{align}
    \mathbb{E}_{\epsilon} \left[ \text{TrueAdv} \right]
    &= \mathbb{E}_{\epsilon} \left[ \mathbb{E}_{\pi_{\hat{f}}} \left[ f\right] - \mathbb{E}_{\pi^{k}} \left[ f\right] \right]
    = \mathbb{E}_{\epsilon} \left[ \int_{\mathcal{A}} \left( \pi_f(a) + \int_{\mathcal{A}} \frac{\delta \pi_f(a)}{\delta \hat{f}(a')} \epsilon(a') \mathrm{d}a'\right) f(a) \mathrm{d}a - \mathbb{E}_{\pi^{k}} \left[ f\right]\right]
    \nonumber\\
    &= \left( \mathbb{E}_{\pi_f} \left[ f\right] - \mathbb{E}_{\pi^k} \left[ f\right] \right) + \mathbb{E}_{\epsilon} \left[ \iint_{\mathcal{A}\times\mathcal{A}} \frac{\delta \pi_f(a)}{\delta \hat{f}(a')} \epsilon(a') f(a) \mathrm{d}a'\mathrm{d}a\right]
    = \mathbb{E}_{\pi_f} \left[ f\right] - \mathbb{E}_{\pi^k} \left[ f\right]
\end{align}
Intuitively, the true advantage term is proportional to the roughness of Q-function ($\Vert \nabla_a f \Vert^2$) while the overestimation bias is proportional to that of the noise ($\Vert \nabla_a \epsilon \Vert^2$). Since the Q-function encapsulates the intrinsic properties of the physical system, it is typically smoother than the noise term. This difference in smoothness, combined with the relative scale of the two functions, determines which of the true advantage and the overestimation bias dominates the perceived policy improvement $\Delta_{\text{on}}$. We formally decompose the two contributing terms as follows.

\begin{definition}
    The characteristic length $\lambda$ of a differentiable function is the ratio of its amplitude to its gradient norm. Specifically, for a deterministic function $f$, the characteristic length is defined under the $L^2$-norm, $\lambda_f = \frac{\Vert f \Vert_{L^2}}{\Vert \nabla f \Vert_{L^2}}$; for a random field $\epsilon$ with kernel $K$, it is defined under expectation, $\lambda_\epsilon = \sqrt{\frac{\mathbb{E} \left[ \vert \epsilon \vert^2 \right]} {\mathbb{E} \left[ \Vert \nabla\epsilon \Vert^2 \right]} }$.
\end{definition}

\begin{assumption}
    The policy improvement follows a policy gradient step $\pi_{\hat{f}}(a|s) = \pi^{k}(a|s) \left(1 + \beta \left( \hat{f}(s,a)-V_{\hat{f}}(s,a)\right) \right)$ where the learning rate $\beta$ is small.
    \label{assumption:policy-gradient-improvement}
\end{assumption}

\begin{assumption}
    In offline-to-online RL, the true Q-function has a much longer characteristic length than the noise $\lambda_f \gg \lambda_\epsilon$, i.e., the Q-function is smoother than the noise.
    \label{assumption:smooth-q-function}
\end{assumption}

\begin{theorem}
    Under Assumption \ref{assumption:policy-gradient-improvement} and Assumption \ref{assumption:smooth-q-function}, the signal-to-noise ratio $\rho = \frac{\mathbb{E}_{\epsilon} \left[ \text{Bias} \right]}{\mathbb{E}_{\epsilon} \left[ \text{TrueAdv} \right]}$ is estimated by
    \begin{align}
        \rho \approx \left( \frac{\mathbb{E} \left[ \epsilon^2 \right]}{\Vert f \Vert^2} \right) \cdot \left( \frac{\lambda_f}{\lambda_\epsilon} \right)^2.
    \end{align}
\end{theorem}

\begin{proof}
With the policy gradient assumption, the derivative of the policy with respect to the empirical Q-function estimator is shown as follows, where $\delta(a)$ is the Dirac delta function.
\begin{align}
    \frac{\delta \pi (a)}{\delta \hat{f}(a')} = \beta \pi^k(a) \left( \delta(a-a') - \frac{\delta V_{\hat{f}}}{\delta \hat{f}(a')}\right) = \beta \pi^k(a) \left( \delta(a-a') - \pi^k(a')\right)
    \nonumber
\end{align}
With this assumption, the overestimation bias and the true advantage are simplified into a concise formulation:
\begin{align}
\begin{cases}
    \mathbb{E}_{\epsilon} \left[ \text{Bias}_{\hat{f}} \right]
    &\approx \beta \displaystyle \iint_{\mathcal{A}\times\mathcal{A}} \pi^k(a) \left( \delta(a-a') - \pi^k(a')\right) K(a,a') \mathrm{d}a'\mathrm{d}a
    \nonumber\\
    &= \beta \left( \mathbb{E}_{a\sim\pi^k} \left[ K(a,a) \right] - \mathbb{E}_{a,a'\sim \pi^k} \left[ K(a,a') \right] \right)
    = \beta \text{Var}_{a\sim\pi^k} \left( \epsilon \right)
    \nonumber\\
    \mathbb{E}_{\epsilon} \left[ \text{TrueAdv} \right]
    &= \displaystyle \int_{\mathcal{A}} \beta \pi^k(a) (f(a) - V_{f}) f(a) \mathrm{d}a = \beta \text{Var}_{a\sim \pi^k} \left( f\right)
\end{cases}
\end{align}
Therefore, the signal-to-noise ratio is estimated by:
\begin{align}
    \rho
    = \frac{\mathbb{E}_{\epsilon} \left[ \text{Bias} \right]}{\mathbb{E}_{\epsilon} \left[ \text{TrueAdv} \right]}
    \approx \frac{\text{Var}_{a\sim\pi^k} \left( \epsilon \right)}{\text{Var}_{a\sim\pi^k} \left( f \right)}
    \approx \frac{\Vert \nabla f(\mathbb{E}[a]) \Vert^2_{\Sigma_\pi}}{\Vert \nabla \epsilon(\mathbb{E}[a]) \Vert^2_{\Sigma_\pi}}
    \approx \frac{\Vert \nabla f(\mathbb{E}[a]) \Vert^2}{\Vert \nabla \epsilon(\mathbb{E}[a]) \Vert^2}
    = \left( \frac{\mathbb{E} \left[ \epsilon^2 \right]}{\Vert f \Vert^2} \right) \cdot \left( \frac{\lambda_f}{\lambda_\epsilon} \right)^2
\end{align}
Here the second approximation uses the first-order Taylor expansion for the functions $f$ and $\epsilon$ at the point $\bar{a}=\mathbb{E}_{a\sim\pi^k}[a]$. The covariance for the exploration policy is $\Sigma_\pi=\mathbb{E}_{a\sim\pi^k} \left[ \left( a-\bar{a} \right) \left( a-\bar{a} \right)^\top \right]$. When $\Sigma_\pi \propto I$, i.e., the exploration is isotropic, the third approximation becomes a strictly-held equality.

In such formulation, the amplitude of noise decays with the number of online training samples increases, while the other three terms are fixed, reflecting the intrinsic physical properties of the environment and the noise of data collection.
\end{proof}

% \begin{align}
%     \Delta_{\text{on}}(s)
%     =& \mathbb{E}_{a\sim\pi} \left[ \hat{f}(s,a)\right] - \mathbb{E}_{a\sim \pi^{k}} \left[ \hat{f}(s,a)\right]
%     = \int_{\mathcal{A}} \pi^k(a|s) \left( 1 + \beta \left( \hat{f}(s,a) - \mathbb{E}_{a\sim \pi^k} \left[ \hat{f}(s,a) \right] \right) \right) \hat{f}(a) \mathrm{d}a - \mathbb{E}_{a\sim\pi^k} [\hat{f}]
%     \nonumber\\
%     =& \beta \left( \mathbb{E}_{a\sim \pi^k} \left[ \hat{f}^2(s,a) \right] - \mathbb{E}_{a\sim \pi^k} \left[ \hat{f}(s,a) \right]^2\right)
%     = \beta \cdot \text{Var}_{\pi^{k}} (\hat{f}).
% \end{align}

\subsection{The Surrogate Objective}

The unbiased estimator for the outer-level objective is Equation \eqref{eq:practical-objective}, which corresponds to the formulation in Equation \eqref{eq:unbiased-estimation}. In this formulation the integral on state $s\in\mathcal{S}$ is restored.
\begin{align}
    \mathcal{J}'_{\text{out}}(g)
    % = \mathbb{E}_{(s,a) \sim d^{\pi}} \left[ \hat{f}^*(s,a)-h^{\pi}(s,a) \sigma^2(s,a) \right]
    = \hat{\mathcal{J}}_{\text{out}} \left( \hat{f}^*, \pi \right) - \mathbb{E}_{\epsilon} \left[ \text{Bias}_{\hat{f}} \right]
    \label{eq:practical-objective}
\end{align}
% 这个估计虽然是无偏的，但是其本身不适合作为MAB的reward，因为它具有较高的方差，这会导致基于采样计算的reward非常不稳定，影响算法整体有效性。因此我们采用类似于advantage的想法，减去一个baseline来减小方差。具体来说，我们选取离线行为策略作为我们的baseline。其直觉在于，训练中的策略在不同的state-action pair上各有优劣，我们希望尽可能把策略本身的好坏和其他环境信息如状态的好坏解耦开来，比如一个较优的策略在较差的state上采样得到了较差的轨迹，而一个较差的策略则可能由于state本身更好而得到更好的结果。
Despite being unbiased, this objective itself is unsuitable as a reward signal for an MAB mainly due to its high variance. Such variance leads to significant instability for a sample-based stochastic reward function, and may influence the effectiveness of the MAB algorithm. To address this issue, we use variance reduction similar to \textit{advantage} function by subtracting a baseline $\pi_\beta$. By subtracting a baseline we decouple the intrinsic policy quality from the inherent value of the state, which ensures that the reward reflects a policy's actual capability rather than the luck of sampling an advantageous state. From this perspective, the offline behavior policy $\pi_\beta$ is a good choice for the baseline because it is usually an ensemble of a large collection of policies, effectively encoding rich environment information.
\begin{align}
    \mathcal{J}'_{\text{out}} (g) = \left( J \left(\pi\left( \hat{f}^* \right)\right) - J(\pi_\beta) \right) - \mathbb{E}_{\epsilon} \left[ \text{Bias}_{\hat{f}} \right]
\end{align}

\paragraph{Approximating the relative performance quality.}
We approximate the term $J(\pi)-J(\pi_\beta)$ with the performance difference lemma \cite{pdl}.
\begin{align}
    J(\pi)-J(\pi_\beta) =& \frac{1}{1-\gamma} \mathbb{E}_{(s,a) \sim d^{\pi}} \left[ A^{\pi_\beta}(s,a) \right]
    \approx \frac{1}{1-\gamma} \mathbb{E}_{s \sim d^{\text{off}}, a\sim \pi} \left[ Q^{\pi_\beta}(s,a) - V^{\pi_\beta}(s) \right]
    \nonumber\\
    \approx& \frac{1}{1-\gamma} \left( \mathbb{E}_{s \sim d^{\text{off}}, a\sim \pi} \left[ \hat{f}(s,a) \right] - \mathbb{E}_{(s,a) \sim d^{\text{off}}} \left[ \hat{f}(s,a) \right] \right) \propto \Delta_{\text{off}}
\end{align}
The first approximation $d^{\pi}(s) \approx d^{\text{off}}(s)$ is mainly a practicality concern to avoid explicitly modeling the density for the online policy, while $d^{\text{off}}$ is a good alternative because of the low uncertainty of $\hat{f}(s,a)$ within its support. The second approximation $Q^{\pi_\beta} \approx f$ utilizes a unique feature of O2O RL that the critic is trained on both online and offline data, ensuring good fitting to offline data distribution.
% Among the readily available data distributions $d^{\text{off}}$ and $d^{\text{on}}$, the former is a natural choice because within this distribution, $\hat{f}(s,a)$ has low uncertainty and is expected to be an accurate estimation of the real Q-function.
These approximations reduce the requirement of estimating the value for $\pi_\beta$ to naively sampling from the offline dataset, allowing for an estimation with low variance and little computational overhead.

\paragraph{Approximating the overestimation bias.}
This approximation of our signal-to-noise ratio $\rho$ implies that at the early stages of training where the amplitude of $\epsilon(s,a)$ is large, the bias term dominates $\Delta_{\text{on}}$, making it a good surrogate for the overestimation bias. Since the early stages of training are exactly when the training undergoes instability due to overestimation, our $\Delta_\text{on}$ is highly effective at preventing unlearning (and the performance dip) during the transition from offline to online phase.

As training progresses and the amplitude of $\epsilon(s,a)$ decays (scaling with $1/\sqrt{N}$), the fidelity of $\Delta_{\text{on}}$ as a bias estimator diminishes. However, in our Multi-Armed Bandit (MAB) setting, this reduction in the signal-to-noise ratio does not compromise the efficacy of the surrogate objective. As the model's performance on offline data improves, the $\Delta_{\text{off}}$ term becomes dominant, effectively rendering $\Delta_{\text{on}}$ a negligible noise component.

In practice, we evaluate the overestimation bias only on the coverage of the online replay buffer. This is because the support of the offline dataset $s \sim d^{\text{off}}$ achieves low uncertainty during the offline pretraining phase, while the area outside existing data coverage has unbounded uncertainty. Then the support of the online replay buffer is natural and the only available choice.

\paragraph{Synthesizing the surrogate.}
With the above approximations, we have a highly symmetric surrogate objective by Equation \eqref{eq:r-q-theory-notation}. Substituting these theoretical notations by practical ones with parameterization, we obtain the surrogate objective proposed in Equation \eqref{eq:R_q}.
\begin{align}
    R_q = \underbrace{\left( \mathbb{E}_{s \sim d^{\text{off}}, a\sim \pi} \left[ \hat{f}(s,a) \right] - \mathbb{E}_{(s,a) \sim d^{\text{off}}} \left[ \hat{f}(s,a) \right] \right)}_{\Delta_\text{off}} - \kappa \underbrace{\left( \mathbb{E}_{s \sim d^{\text{on}}, a\sim \pi} \left[ \hat{f}(s,a) \right] - \mathbb{E}_{(s,a) \sim d^{\text{on}}} \left[ \hat{f}(s,a) \right] \right)}_{\Delta_\text{on}}
    \label{eq:r-q-theory-notation}
\end{align}
\section{Discussions and Details on the Example MDP} \label{Appendix: simple MDP}

\begin{figure*}[ht]
\centering
  \begin{minipage}[c]{0.35\linewidth}
    \centering
    \includegraphics[width=\linewidth]{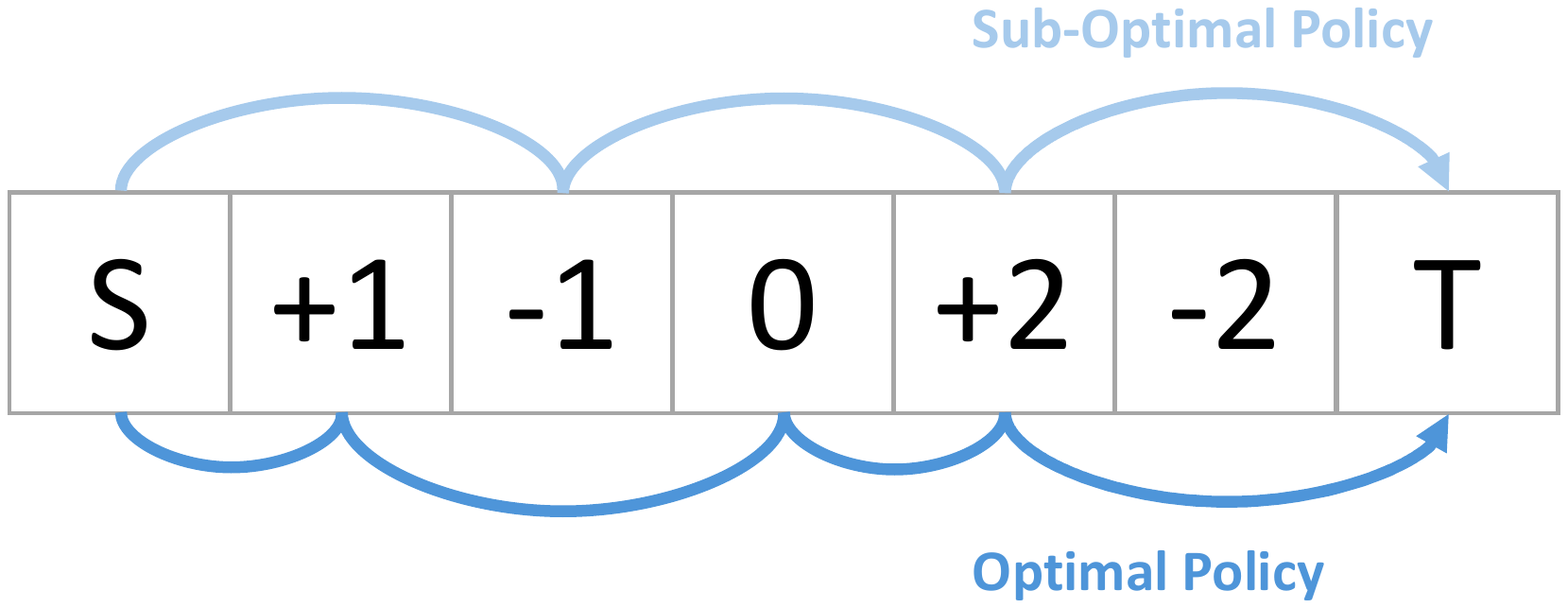}
    % \caption{First Image}
    % \label{fig:first-image}
  \end{minipage}
  \begin{minipage}[c]{0.25\linewidth}
    \centering
    \begin{subfigure}{\linewidth}
      \centering      \includegraphics[width=\linewidth]{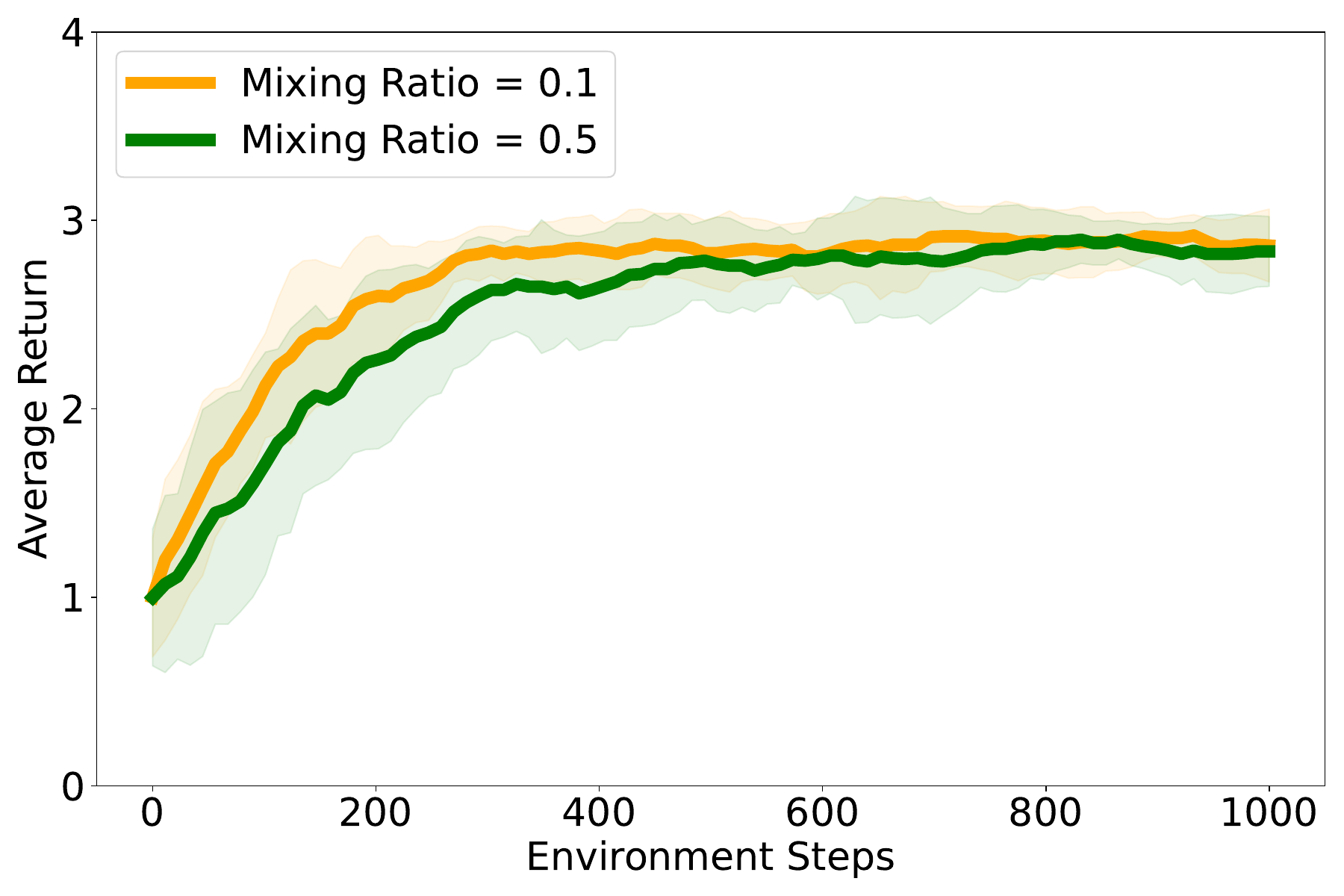}
      % \caption{Second Image}
      % \label{fig:second-image}
    \end{subfigure}
    \begin{subfigure}{\linewidth}
      \centering
      \includegraphics[width=\linewidth]{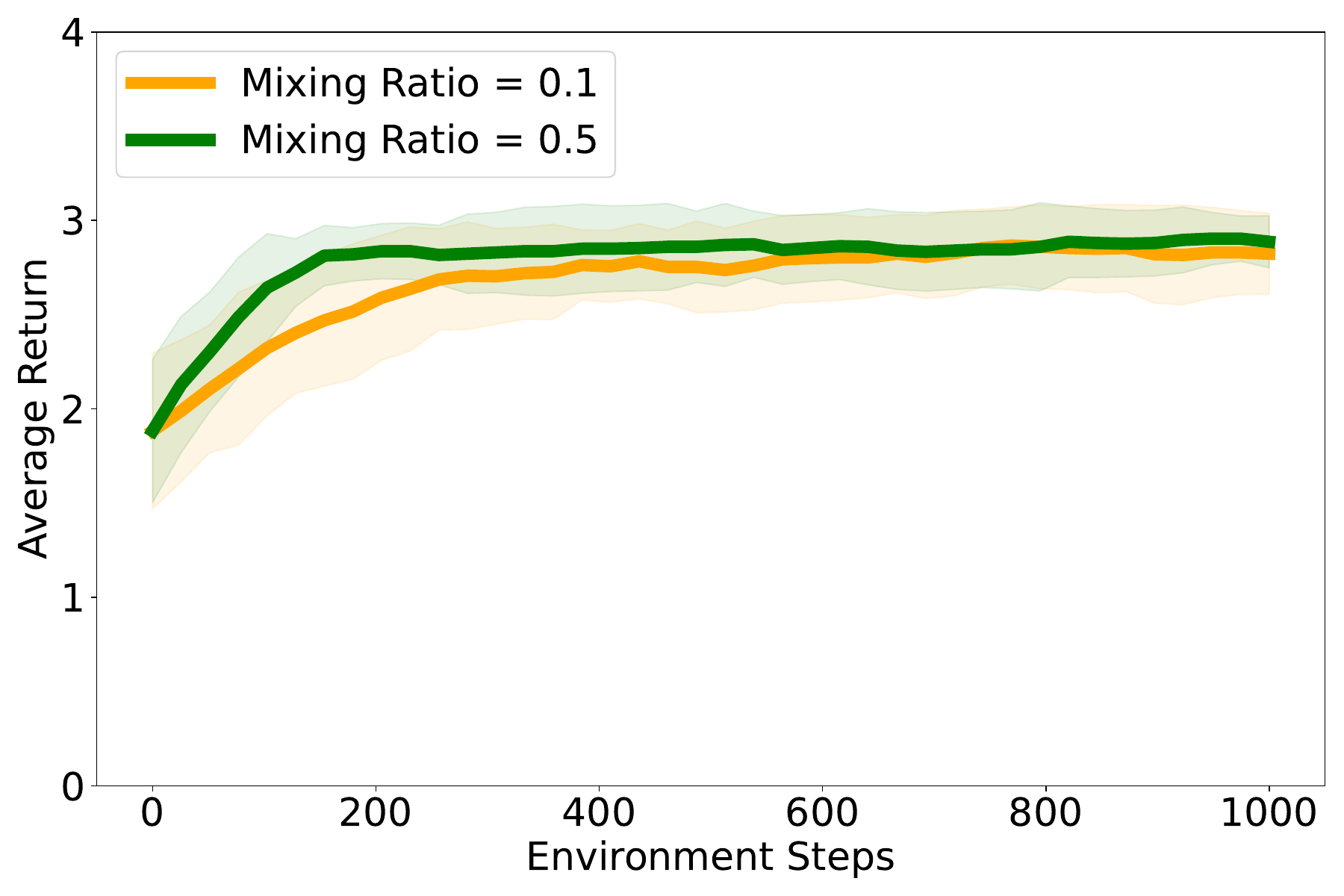}
      % \caption{Third Image}
      % \label{fig:third-image}
    \end{subfigure}
  \end{minipage}
  \begin{minipage}[c]{0.35\linewidth}
    \centering
    \begin{subfigure}{\linewidth}
      \centering
      \includegraphics[width=\linewidth]{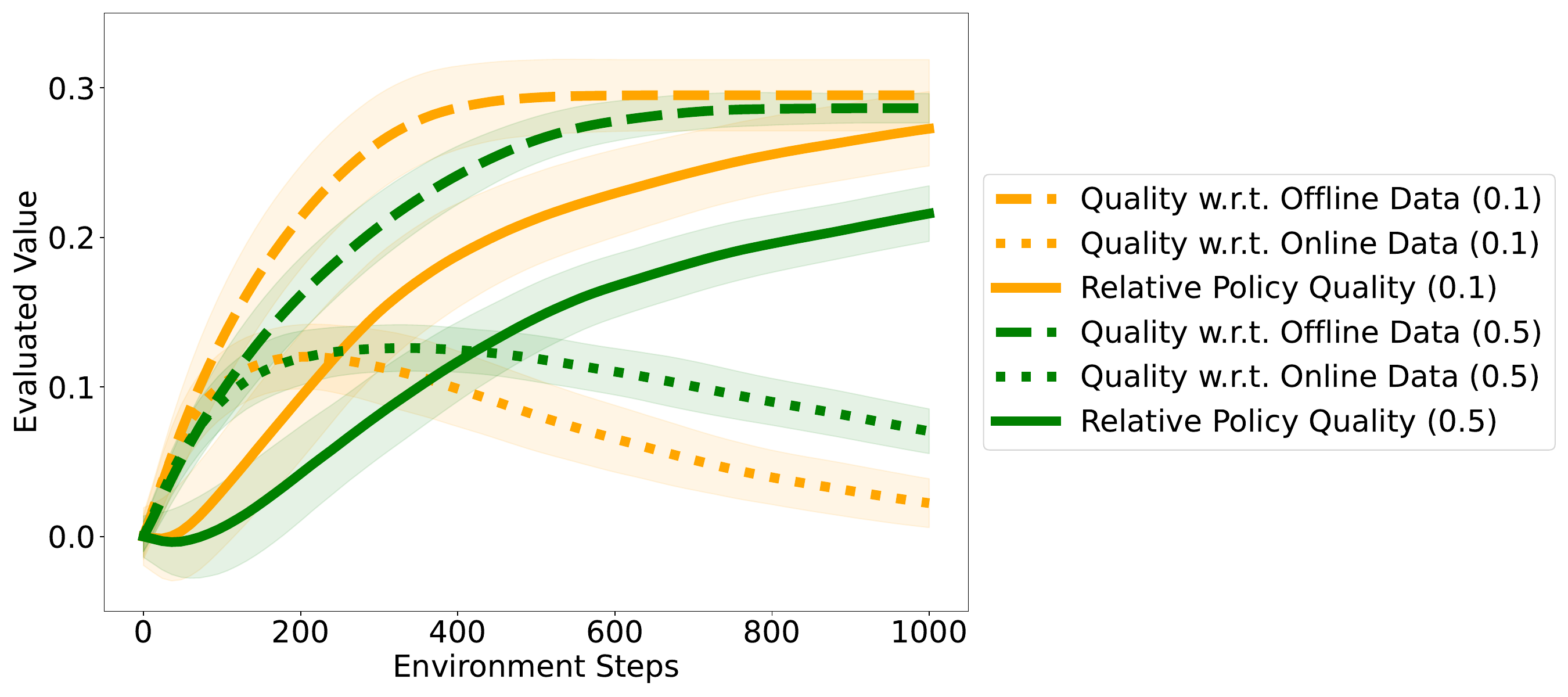}
      % \caption{Fourth Image}
      \label{fig:fourth-image}
    \end{subfigure}
    \begin{subfigure}{\linewidth}
      \centering
      \includegraphics[width=\linewidth]{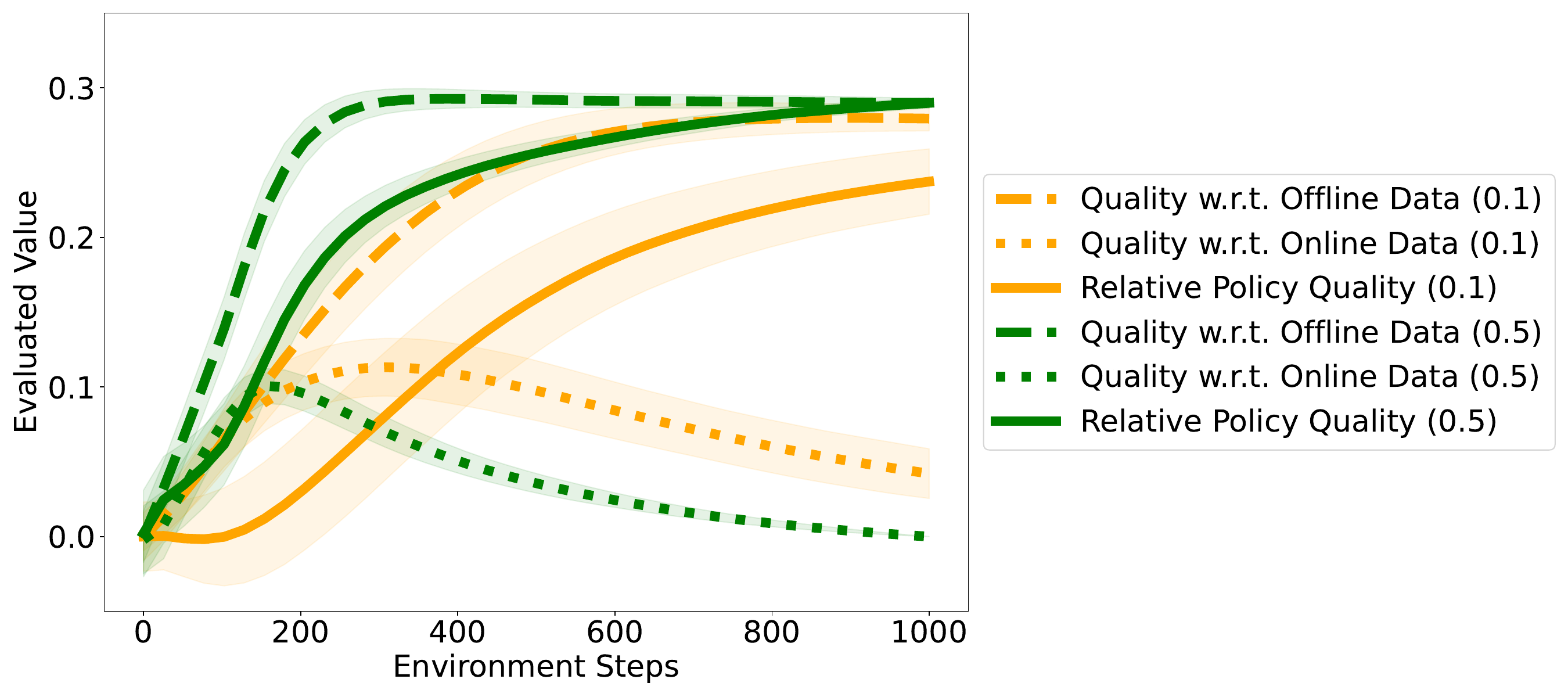}
      % \caption{Fifth Image}
      \label{fig:fifth-image}
    \end{subfigure}
  \end{minipage}
  \caption{An example of offline-to-online RL with various dataset qualities and the evaluated policy quality. \textbf{Left:} Illustrates a simple MDP with an example of sub-optimal/optimal policies. An agent initiates from the start state $\mathbf{S}$, with options to advance either 1 or 2 steps towards the terminal state $\mathbf{T}$ in each timestep. \textbf{Right:} Demonstrates the average return and perceived quality of the current policy by the agent against offline and online data under mixing ratios of 0.1 and 0.5 with lower-quality offline data generated with sub-optimal policy (\textbf{upper}) and higher-quality offline data generated with an equal blend of sub-optimal policy and optimal policy (\textbf{lower}).}
  \label{fig: example in appendix}
\end{figure*}

% In this section, we delve deeper into the analysis and elucidation of the Markov Decision Process (MDP) example utilized in our paper.
In this section, we analyze the example MDP that is used to analyze the training dynamics of $R_q$ in our paper.
The MDP model we have specified (as depicted in Figure~\ref{fig: example in appendix}, left) facilitates the straightforward construction of both optimal and sub-optimal policies. These policies, in turn, enable the generation of offline data of varying quality for use in offline-to-online reinforcement learning. We demonstrate the application of lower-quality offline data constructed using a sub-optimal policy for offline-to-online RL (upper right), as well as higher-quality offline data created using both sub-optimal and optimal policies (lower right). The primary rationale behind employing offline data of these two distinct quality levels is to more accurately reflect real-world scenarios where offline data may vary in quality. Typically, offline data encompasses not only expert data but also includes sub-optimal data. For each data quality type, we engage in 1,000 steps of offline training followed by 1,000 steps of online training. Within the online training phase, we explore the effects of incorporating less versus more offline data replay by selecting mixing ratios of 0.1 and 0.5, respectively. As to the agent for this simple environment, we utilize a Q-Learning Agent, a fundamental reinforcement learning model known for its simplicity and effectiveness. This agent learns optimal actions within an environment by updating a Q-table, which estimates rewards for actions in various states, guided by learning rate 1e-3 and a discount factor of 0.9 for future rewards.

Our observations reveal that in the context of lower-quality offline data, opting for a relatively smaller mixing ratio in this environment facilitates a quicker enhancement in performance. Correspondingly, when evaluating policy quality with respect to both offline and online data, we note a consistent improvement in policy quality relative to offline data across different mixing ratios. A smaller mixing ratio, enabling more online exploration, appears to more effectively enhance this measure. As for the quality relative to online data, we observe an overall trend of initial increase followed by a decrease, aligning with the distribution shift occurring as online training progresses. Overall, a smaller mixing ratio results in a more steady improvement in relative policy quality. Conversely, with higher-quality offline data, a relatively larger mixing ratio yields better outcomes, evident not only in faster performance enhancement but also in a more consistent uplift in relative policy quality.

Furthermore, compared to assessing policy quality based solely on offline or online data, our employed measure of relative data quality more accurately reflects the overall trend of an agent's policy in relation to both types of data. This makes it an excellent metric for evaluating data replay schemes throughout the training process. This enriched discussion aims to provide a clearer understanding of how different qualities of offline data and varying mixing ratios impact the efficiency of offline-to-online RL training strategies, showcasing the importance of tailored data replay schemes in optimizing agent performance.

\section{Environment Details} \label{appendix: environment details}

\subsection{Antmaze Tasks}

Antmaze tasks challenge an agent to guide an 8-DoF ant robot through a maze, starting from one point and ending at a set goal. A simple reward system is used: the agent gets a reward of $+1$ if the goal is reached and $0$ otherwise. There are three maze levels: large, medium and umaze with diverse or play datasets. While the ``diverse'' dataset is filled with trajectories from arbitrary start points to random goals, the ``play'' dataset has trajectories directed at specific, non-goal locations. Each task ran for 1000 episodes. All the agents were initially trained with the offline datasets for $1 \mathrm{M}$ steps. Subsequent online fine-tuning was conducted over $1 \mathrm{M}$ environment steps, integrating ROAD or competing offline data replay strategies.

\subsection{Locomotion Tasks}

This set of benchmarks is pivotal for assessing the adaptability and efficacy of offline-to-online reinforcement learning algorithms. The D4RL locomotion domain encompasses several tasks, including HalfCheetah, Walker2d, and Hopper. Each of these tasks is designed to test the agility and control of simulated robotic models in executing complex movements such as running, walking, or hopping in a physically accurate simulation environment. Crucially, for each of these tasks, the D4RL benchmark provides offline datasets of varying quality levels to simulate different training scenarios and challenges. These quality levels are categorized into five distinct types: expert, medium-expert, medium, medium-replay, and random. These datasets enable researchers to rigorously test and benchmark the performance of reinforcement learning algorithms across a spectrum of real-world challenges. By leveraging these varied quality levels of offline data, the adaptability, learning efficiency, and robustness of algorithms like ROAD can be effectively evaluated, offering insights into their potential applicability in diverse reinforcement learning scenarios.

\subsection{Franka Kitchen Tasks}

The tasks within the Franka Kitchen domain task an agent with manipulating a 9-DoF Franka robot to set up a kitchen based on a specific layout. The task breaks down into 4 distinct subtasks. The agent earns rewards on a scale from $0$ to $+4$, based on the number of subtasks successfully executed. Completing the overall task demands mastering individual subtasks and deducing the right sequence. We evaluated agents on this domain utilizing two contrasting datasets: kitchen-complete and kitchen-mixed. The ``complete'' dataset encompasses full task sequences, while the ``mixed'' dataset presents fragmented sequences without a complete demonstration, pushing agents' generalization abilities to the limit. Each task here also had a length of 1000 episodes. Online fine-tuning was executed across $1 \mathrm{M}$ environment steps after $1 \mathrm{M}$ offline pretraining steps, deploying either ROAD or other alternative offline replay strategies.

% \subsection{Adroit Tasks}

% The Adroit arena demands control over a sophisticated 24-DoF shadow hand robot. The chosen tasks for this domain were relocate-binary (listed twice, assuming a typo). These tasks revolve around a limited set of expert human data (roughly 25 sets), bolstered by trajectories from a behavior-cloned policy. Given the tight dataset distribution and expansive action space, the Adroit tasks are inherently challenging due to sparse rewards and exploration issues. During the offline training phase, agents underwent training for $20 \mathrm{~K}$ steps. Online fine-tuning spanned $1 \mathrm{M}$ environment steps for both tasks. Both tasks had a 200-episode duration.

\begin{figure*}[ht]
\centering
  \begin{minipage}[c]{0.3\linewidth}
    \centering
    \begin{subfigure}{\textwidth}
      \centering      \includegraphics[width=\linewidth]{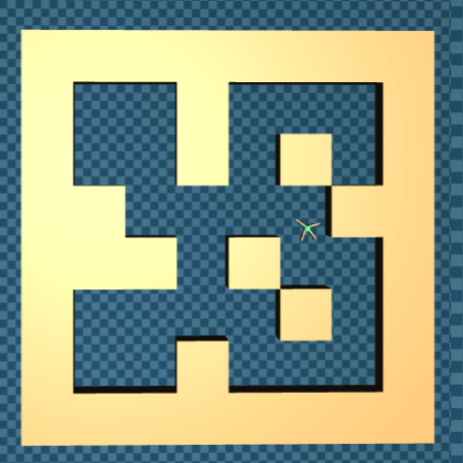}
      % \caption{Second Image}
      % \label{fig:second-image}
    \end{subfigure}
  \end{minipage}
  \begin{minipage}[c]{0.3\linewidth}
    \centering
    \begin{subfigure}{\textwidth}
      \centering      \includegraphics[width=\linewidth]{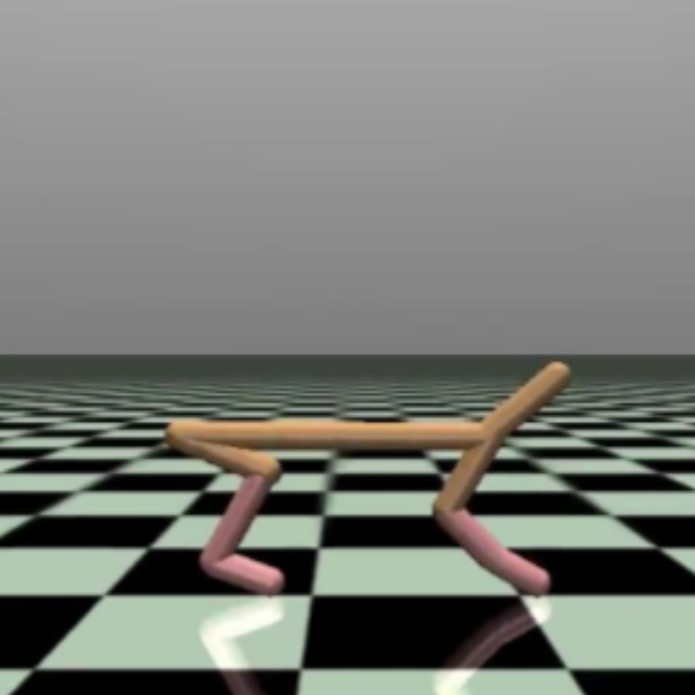}
      % \caption{Second Image}
      % \label{fig:second-image}
    \end{subfigure}
  \end{minipage}
  \begin{minipage}[c]{0.3\linewidth}
    \centering
    \begin{subfigure}{\textwidth}
      \centering      \includegraphics[width=\linewidth]{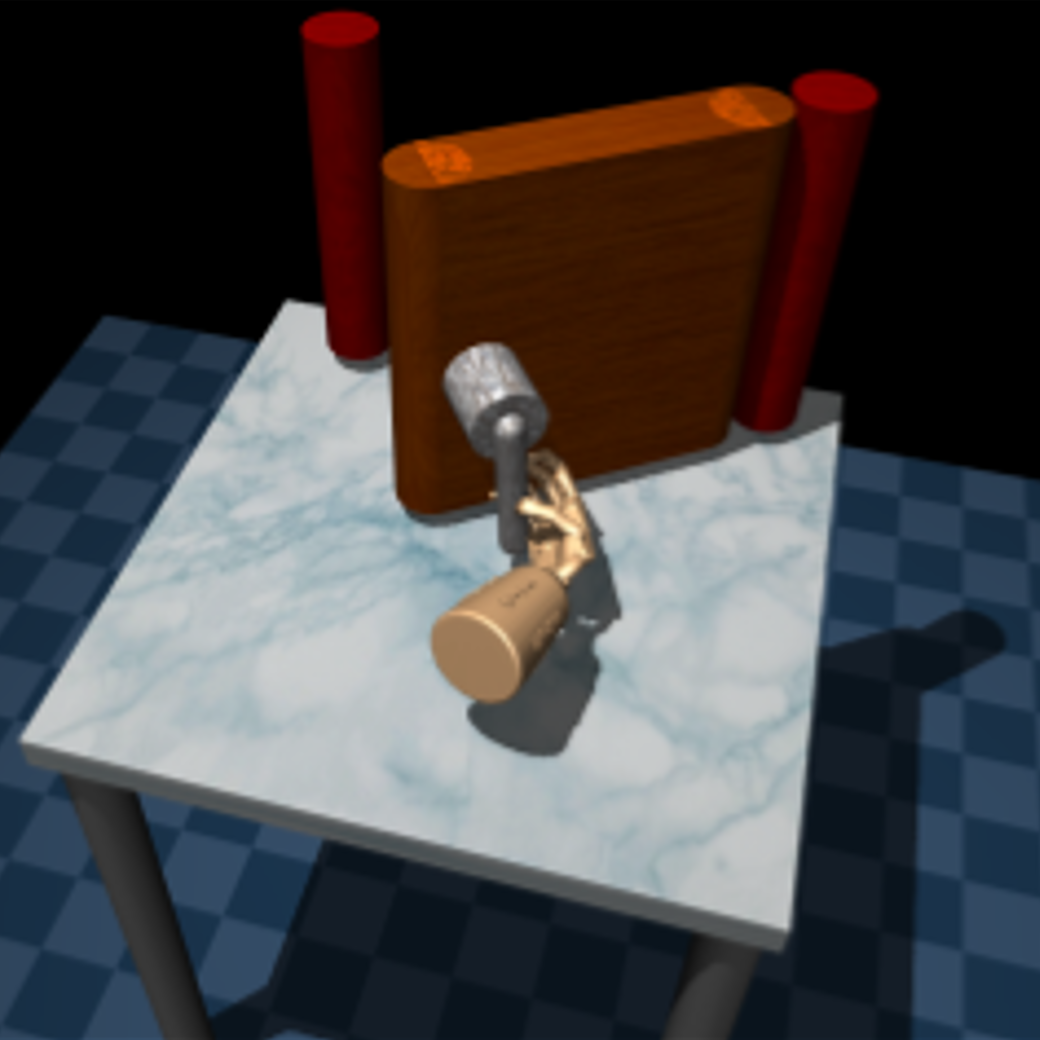}
      % \caption{Second Image}
      % \label{fig:second-image}
    \end{subfigure}
  \end{minipage}
  % \begin{minipage}[c]{0.3\linewidth}
  %   \centering
  %   \begin{subfigure}
  %     \centering
  %     \includegraphics[width=\linewidth]{figures/envs/env1.png}
  %     % \caption{Fourth Image}
  %     \label{fig:fourth-image}
  %   \end{subfigure}
  %   \begin{subfigure}
  %     \centering
  %     \includegraphics[width=\linewidth]{figures/ICML/AppendixA/qdiff-high.pdf}
  %     % \caption{Fifth Image}
  %     \label{fig:fifth-image}
  %   \end{subfigure}
  % \end{minipage}
  \caption{A selection of benchmark tasks in Antmaze Task Domain (left), Locomotion Task Domain (middle), and Adroit Task Domain (right).}
  \label{fig: appendix env introduction}
\end{figure*}

% \section{Combining ROAD with PEX} \label{Appendix: result of PEX}

% In the main paper, we presented the implementation of ROAD with IQL. Here, we extend our evaluation to include the integration of ROAD with PEX \cite{apl}, with the results displayed in Table \ref{tab:performance_metrics}. As can be seen, when combined with PEX, ROAD achieves generally favorable outcomes across various environments and offline datasets, surpassing baseline performances in the majority of tasks. Compared to these baselines, ROAD exhibits superior adaptability across diverse settings and datasets.

\begin{table}[ht]
\centering
\small
\setlength{\tabcolsep}{2.5pt}
\caption{Performance of ROAD and the baseline methods for offline data replay under Antmaze tasks, Locomotion tasks and Kitchen tasks with PEX. All results are assessed across 4 random seeds. The best performance for fixed mixing ratios is \underline{underlined}, and the best-performed score is \textbf{bolded}. \\ \ }
\label{tab:performance_metrics}
\begin{tabular}{lcccccccccc}
\toprule
\multicolumn{1}{c}{} & \multicolumn{6}{c}{Fixed Mixing Ratios} & \multicolumn{1}{c}{} & \multicolumn{1}{c}{} & \multicolumn{1}{c}{} & \multicolumn{1}{c}{} \\
\cmidrule(lr){2-7}
Tasks & $0.0$ & 0.1 & 0.2 & 0.3 & 0.4 & 0.5 & Uniform & Decreasing & BR & ROAD \\
\midrule
antmaze-large-diverse & \underline{70.15} & 69.48 & 66.97 & 62.51 & 65.98 & 69.83 & 64.47 & 69.99 & 33.51 & \textbf{74.51} \\
antmaze-large-play    & 64.33 & 66.66 & \underline{68.34} & 67.18 & 67.01 & 64.18 & 64.75 & 69.01 & 64.00 & \textbf{69.13} \\
antmaze-medium-diverse& 87.66 & 86.33 & \underline{89.82} & 86.82 & 86.67 & 86.49 & 91.50 & \textbf{93.25} & 88.50 & 92.00 \\
antmaze-medium-play   & 87.00 & 88.83 & 86.83 & \underline{89.33} & 85.49 & 86.66 & 84.75 & 84.99 & 89.25 & \textbf{89.67} \\
antmaze-umaze-diverse & \textbf{\underline{95.16}} & 86.99 & 81.01 & 43.00 & 6.17 & 8.68 & 42.22 & 51.75 & 83.50 & 86.35 \\
antmaze-umaze         & 97.99 & \underline{99.00} & 97.66 & 98.00 & 97.16 & 97.34 & 96.25 & 97.00 & 96.75 & \textbf{99.83} \\
% \textbf{Antmaze Total}       & 502.29 & 497.29 & 490.63 & 446.84 & 408.48 & 413.18 & 443.94 & 465.99 & 455.51 & 0.00 \\
\midrule
halfcheetah-random    & 78.07 & 75.82 & \underline{80.43} & 78.16 & 79.11 & 76.65 & 80.07 & 87.35 & 76.94 & \textbf{88.36} \\
halfcheetah-medium-replay    & \textbf{\underline{78.73}} & 74.20 & 73.44 & 72.68 & 64.26 & 60.32 & 71.28 & 62.57 & 72.51 & 77.67 \\
halfcheetah-medium    & 81.44 & \underline{85.30} & 82.71 & 83.27 & 73.39 & 78.48 & 83.53 & 78.86 & 83.37 & \textbf{86.94} \\
halfcheetah-medium-expert    & 88.20 & 89.28 & 89.44 & \underline{89.90} & 88.35 & 87.35 & 88.14 & 89.68 & 91.08 & \textbf{92.10} \\
halfcheetah-expert    & 93.48 & \underline{94.44} & 93.75 & 93.61 & 89.55 & 90.73 & 93.17 & 93.86 & 96.28 & \textbf{96.86} \\
% \textbf{Halfcheetah Total} & 0.00 & 0.00 & 0.00 & 0.00 & 0.00 & 0.00 & 0.00 & 0.00 & 0.00 & 0.00 \\
\midrule
walker2d-random       & 27.14 & \underline{28.20} & 16.54 & 15.77 & 12.06 & 13.84 & 16.82 & 15.93 & 31.63 & \textbf{32.08} \\
walker2d-medium-replay       & 36.08 & 36.24 & 34.83 & \underline{38.89} & 34.85 & 28.39 & 35.95 & 34.47 & 42.21 & \textbf{53.50} \\
walker2d-medium       & 39.66 & 41.55 & \underline{50.49} & 40.42 & 38.86 & 34.01 & 40.51 & 48.84 & 37.47 & \textbf{63.63} \\
walker2d-medium-expert       & 57.74 & 61.69 & 41.67 & 66.36 & \underline{70.61} & 53.73 & 66.29 & 67.07 & 23.03 & \textbf{79.04} \\
walker2d-expert       & 36.99 & 46.23 & 37.11 & 37.64 & 49.26 & \underline{50.31} & 48.57 & 52.22 & \textbf{89.06} & 79.63 \\
% \textbf{Walker2d Total}& 0.00 & 0.00 & 0.00 & 0.00 & 0.00 & 0.00 & 0.00 & 0.00 & 0.00 & 0.00 \\
\midrule
hopper-random         & 29.82 & 20.79 & 13.49 & 21.65 & \textbf{\underline{36.61}} & 25.20 & 22.67 & 23.01 & 17.81 & 34.46 \\
hopper-medium-replay         & 38.44 & 55.97 & 42.60 & 40.81 & 45.73 & \underline{55.99} & 52.17 & 49.10 & 47.30 & \textbf{68.43} \\
hopper-medium         & \underline{40.29} & 34.55 & 35.98 & 33.57 & 39.23 & 34.07 & 42.27 & 39.33 & 41.43 & \textbf{58.61} \\
hopper-medium-expert         & 42.45 & \underline{66.13} & 50.27 & 49.53 & 58.79 & 56.38 & 39.78 & 50.20 & 48.83 & \textbf{70.81} \\
hopper-expert         & 58.57 & \underline{62.31} & 40.26 & 41.82 & 43.52 & 43.06 & 60.24 & \textbf{64.74} & 49.84 & 57.58 \\
% \textbf{Hopper Total}  & 0.00 & 0.00 & 0.00 & 0.00 & 0.00 & 0.00 & 0.00 & 0.00 & 0.00 & 0.00 \\
\midrule
kitchen-partial       & 15.42 & 26.66 & \underline{39.17} & 17.92 & 33.32 & 19.16 & 18.80 & 0.00 & 2.48 & \textbf{46.65} \\
kitchen-mixed         & 0.00 & 19.61 & 37.11 & 42.06 & \underline{50.00} & 48.71 & 18.10 & 42.52 & 21.26 & \textbf{50.74} \\
kitchen-complete      & 1.25 & \underline{47.07} & 15.82 & 16.66 & 40.85 & 33.74 & 0.62 & 34.95 & 51.25 & \textbf{58.34} \\
% \textbf{Kitchen Total}& 0.00 & 0.00 & 0.00 & 0.00 & 0.00 & 0.00 & 0.00 & 0.00 & 0.00 & 0.00 \\
\midrule
\textbf{Average} & 56.09 & 60.97 & 56.91 & 55.32 & 56.53 & 54.30 & 55.12 & 58.36 & 57.47 & \textbf{71.12} \\
\bottomrule
\end{tabular}
\end{table}

\section{Combining ROAD with CQL, Cal-QL and Proto} \label{Appendix: result of Cal-QL and CQL}

In the main paper, we presented the performance of ROAD (Reinforcement learning Online-Offline Adaptive Decision-making) alongside baseline methods under IQL (Implicit Q-Learning) \cite{iql}. To offer a more comprehensive evaluation of ROAD's effectiveness, we also tested it with PEX (Policy Expansion), Cal-QL (Calibrated Q-Learning) and CQL (Conservative Q-Learning), serving as the backbone algorithms for transitioning from offline to online reinforcement learning. CQL is an advanced reinforcement learning algorithm designed to address the challenge of distribution shift between offline and online data, which is a common hurdle in offline-to-online RL scenarios. CQL operates by explicitly minimizing a lower bound of the expected return for the policy it learns, while simultaneously preventing the overestimation of Q-values for unseen state-action pairs. This conservative approach to estimating Q-values helps to mitigate the risk of erroneous policy evaluation and decision-making based on the limited or biased offline data. The key strength of CQL lies in its ability to cautiously navigate the exploration-exploitation trade-off in environments where collecting new data may be expensive or impractical. On the other hand, Cal-QL takes a slightly different approach to tackling the offline-to-online transition challenge. It focuses on calibrating the Q-function estimates to be more aligned with the true action values observed in the offline data set. By adjusting the Q-values to reflect more realistic outcomes, Cal-QL aims to improve the policy's performance by ensuring that the learning process is grounded in the empirical realities of the environment. This calibration process helps in reducing the overoptimism often associated with conventional Q-learning algorithms, especially in contexts where the offline data might not fully represent the complexities of the target environment.

Our experiments were conducted on both Antmaze and Adroit tasks, with the results illustrated in Table \ref{tab:performance_metrics} (PEX) and Figures \ref{fig: all Cal-QL} and \ref{fig: all CQL}. It was observed that fixed mixing ratios exhibit a certain level of instability in these cases. For instance, a mixing ratio fixed at 0.3 shows promising overall results with Cal-QL across both types of environments, yet the same setting performs significantly weaker with CQL. Similarly, strategies such as decreasing mixing ratios and uniform selection among mixing ratios demonstrated commendable performance on certain tasks but were unable to consistently outperform the best fixed mixing ratio. Compared to the baselines, ROAD achieved superior performance across various backbone offline-to-online RL algorithms and different task domains. It was able to match or even surpass the fixed optimal mixing ratio, further showcasing ROAD's adaptability to diverse settings. This comprehensive evaluation highlights ROAD's robustness and adaptability, indicating its potential to enhance learning efficiency and effectiveness in a broad range of reinforcement learning applications. 

% \begin{figure*}[ht]
% \centering
%   \begin{minipage}[c]{\textwidth}
%     \centering
%     \begin{subfigure}{0.5\textwidth}
%       \centering      
%     \includegraphics[width=0.45\textwidth]{figures/ICML/MainResult/Cal-QL/Cal-QL-all-antmaze.pdf}
%       % \label{fig:second-image}
%     \end{subfigure}
%     \begin{subfigure}{0.5\textwidth}
%       \centering
%       \includegraphics[width=0.45\textwidth]{figures/ICML/MainResult/Cal-QL/Cal-QL-all-adroit.pdf}
%       % \label{fig:second-image}
%     \end{subfigure}
%   \end{minipage}
%   \caption{Normalized Returns of ROAD and the baseline methods with Cal-QL \cite{nakamoto2023cal} as the backbone in Antmaze and Adroit tasks.}
%   \label{fig: all Cal-QL}
% \end{figure*}

\begin{figure*}[ht]
\centering
\begin{minipage}{\textwidth}
    \centering
    \begin{subfigure}{0.5\textwidth}
        \centering
        \includegraphics[width=\textwidth]{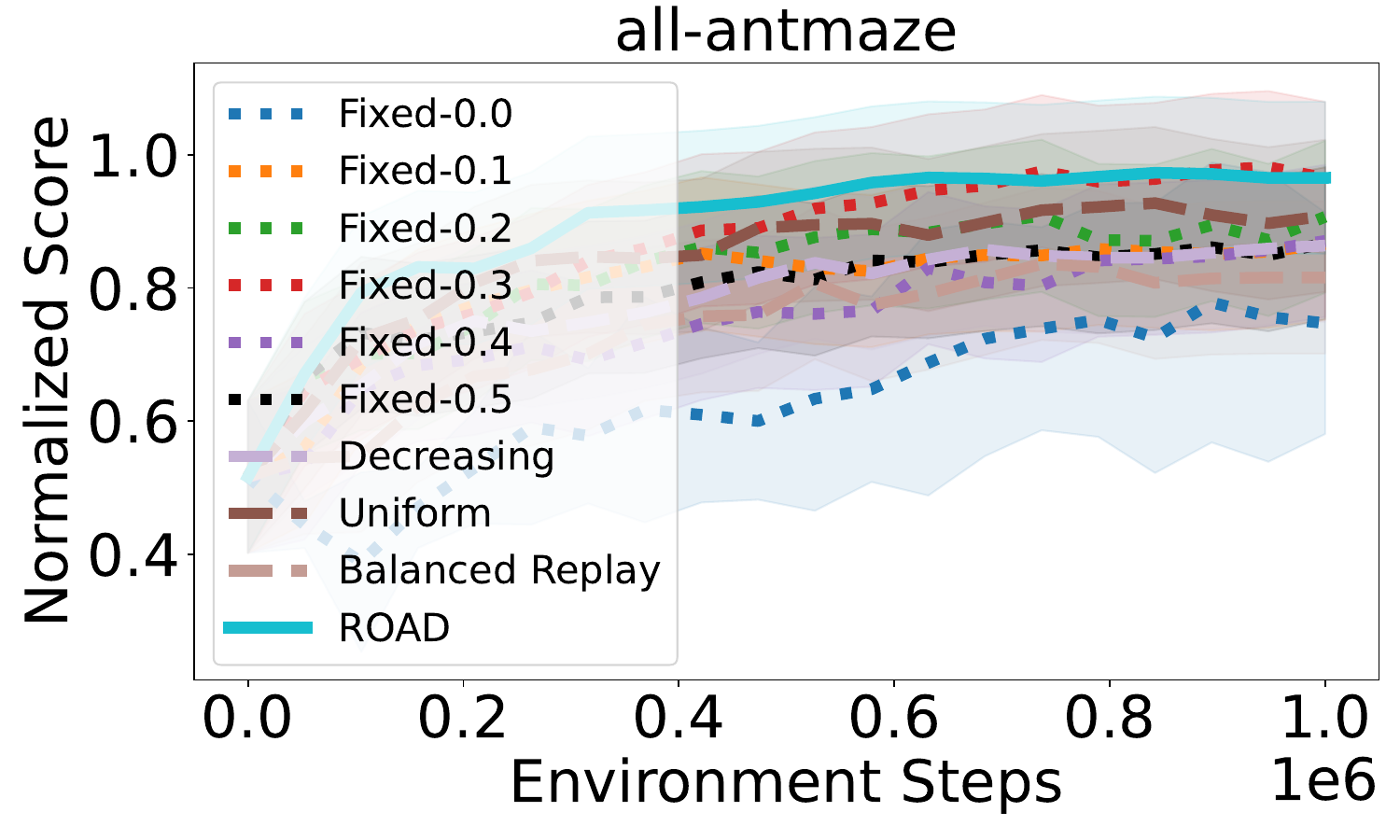}
        % \label{fig:antmaze}
    \end{subfigure}%
    \begin{subfigure}{0.5\textwidth}
        \centering
        \includegraphics[width=\textwidth]{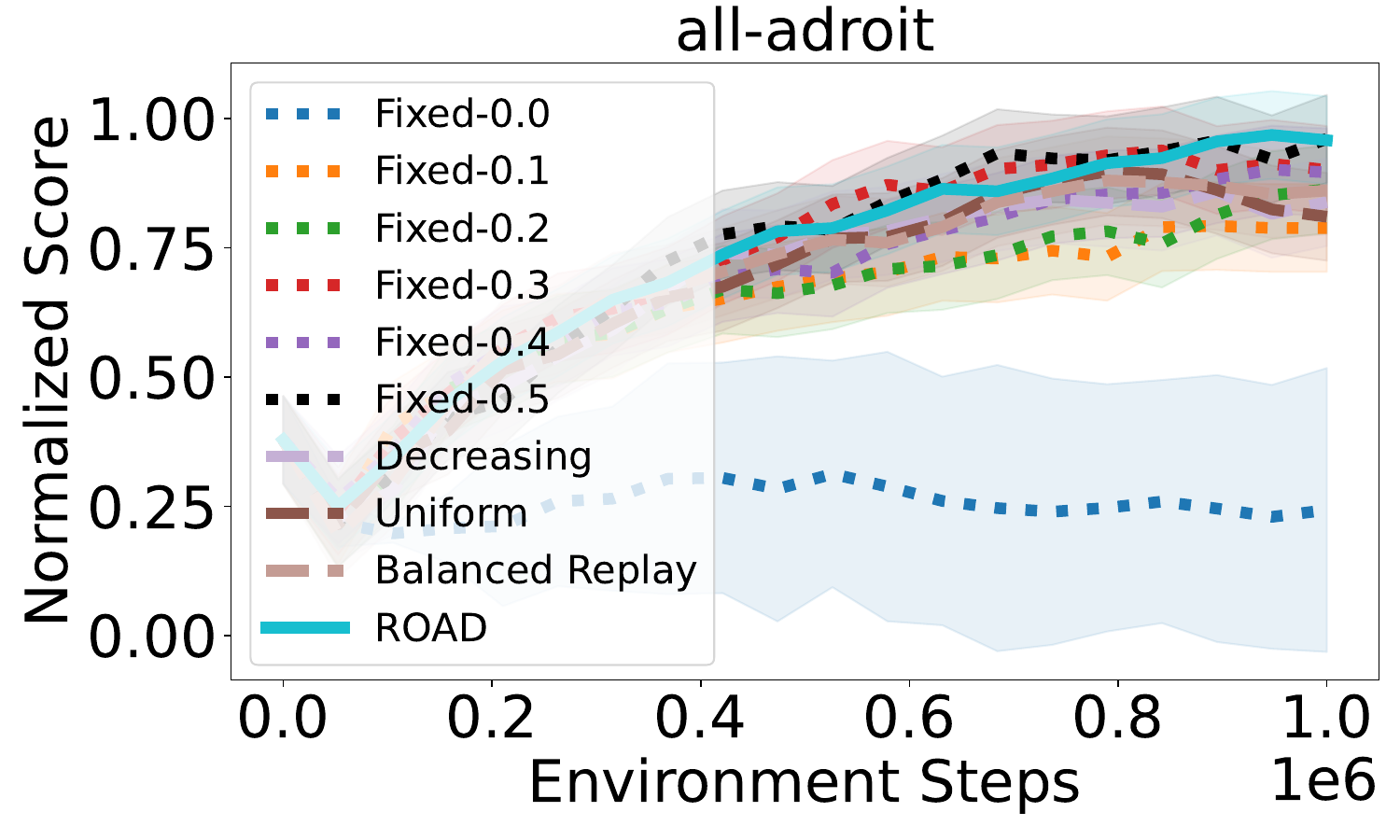}
        % \label{fig:adroit}
    \end{subfigure}
\end{minipage}
\caption{Normalized Returns of ROAD and the baseline methods with Cal-QL \protect\cite{cal-ql} as the backbone in Antmaze and Adroit tasks.}
\label{fig: all Cal-QL}
\end{figure*}

\begin{figure*}[ht]
\centering
\begin{minipage}{\textwidth}
    \centering
    \begin{subfigure}{0.5\textwidth}
        \centering
        \includegraphics[width=\textwidth]{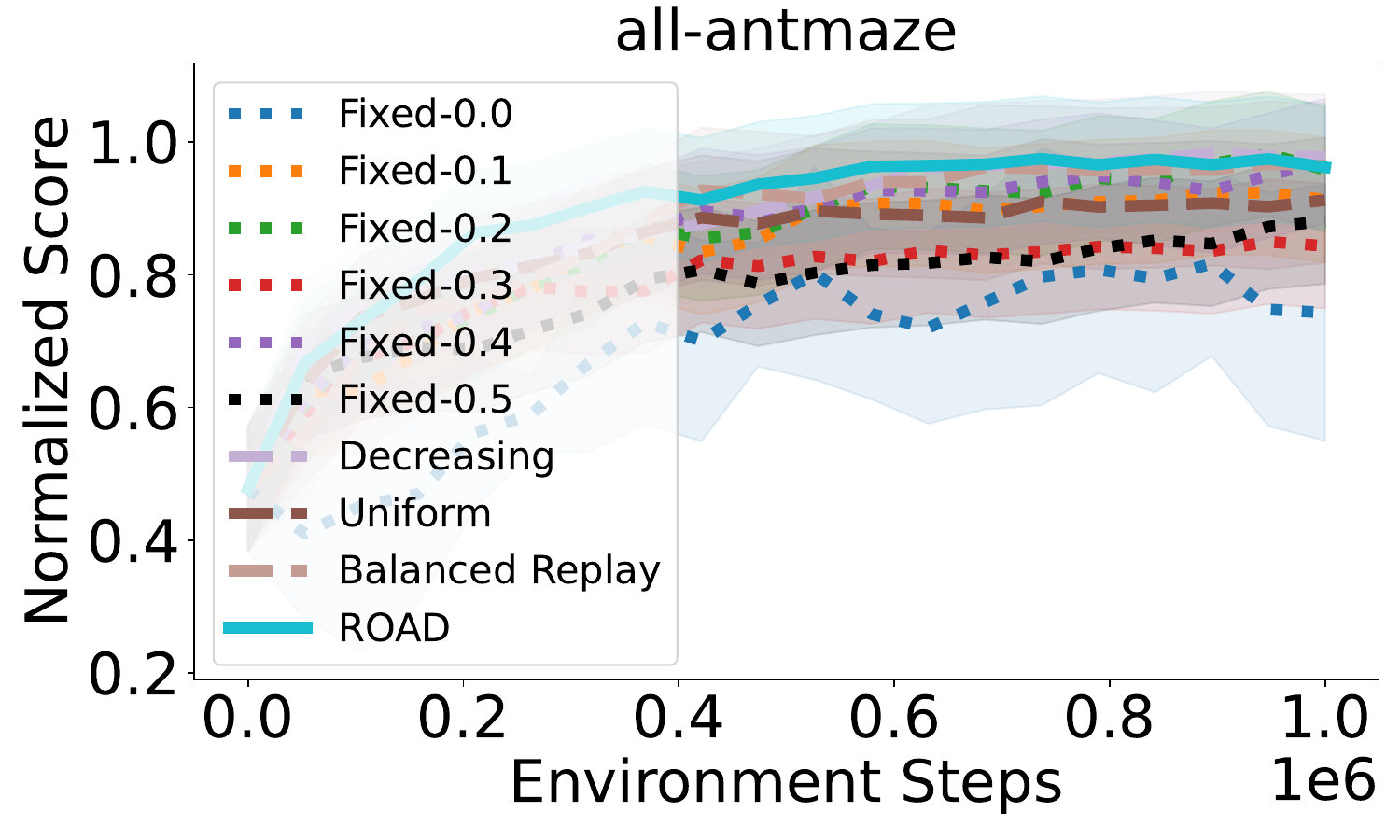}
        % \label{fig:antmaze}
    \end{subfigure}%
    \begin{subfigure}{0.5\textwidth}
        \centering
        \includegraphics[width=\textwidth]{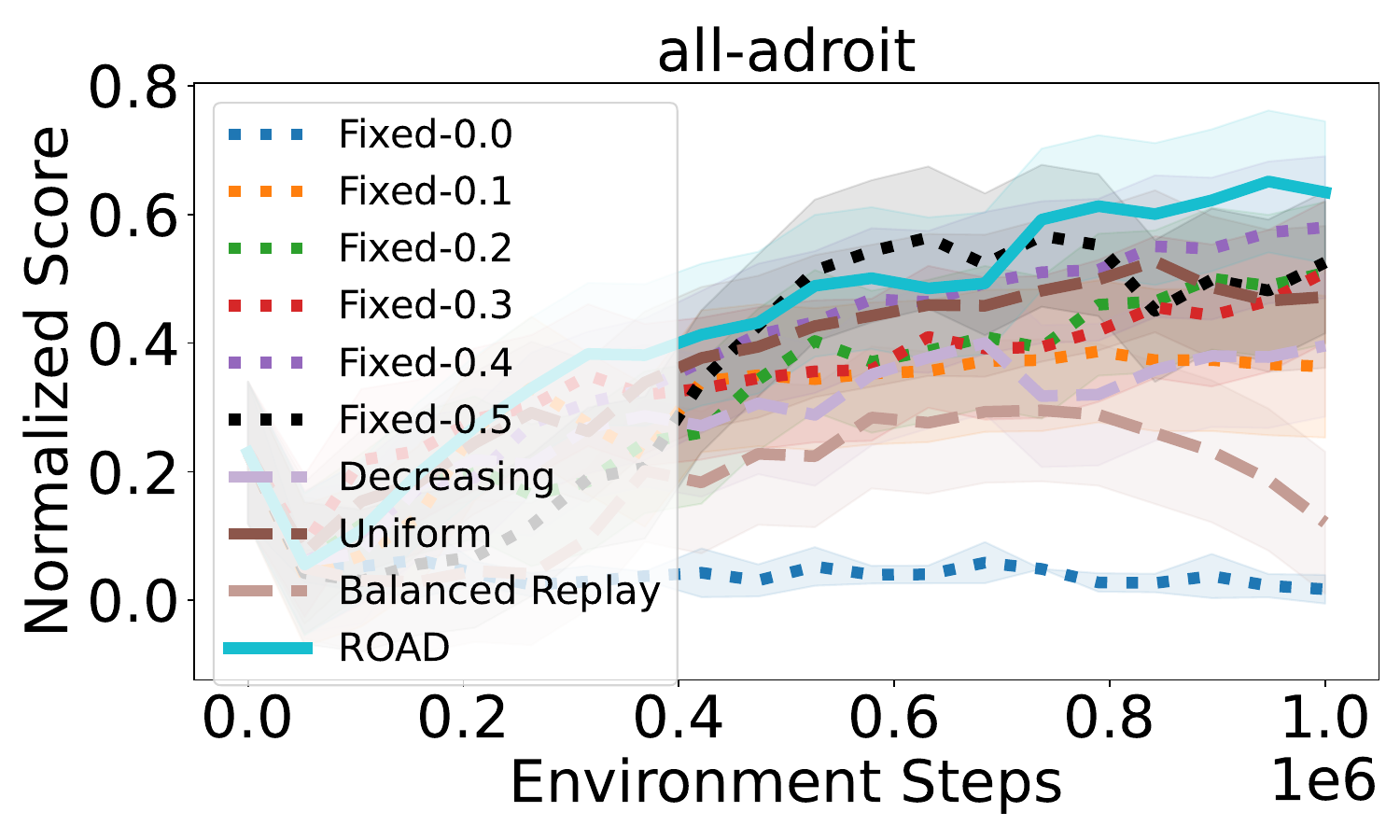}
        % \label{fig:adroit}
    \end{subfigure}
\end{minipage}
\caption{Normalized Returns of ROAD and the baseline methods with CQL \protect\cite{cql} as the backbone in Antmaze and Adroit tasks.}
\label{fig: all CQL}
\end{figure*}

Besides, we have also conducted experiments combining ROAD with Proto \cite{proto}, an offline2online RL algorithm which uses a 50\% fixed mixing ratio (symmetric sampling strategy) in its initial implementation. We examined Proto under different fixed mixing ratios and evaluated its performance when combined with ROAD. Our results on the Halfcheetah domain, across random, medium-replay, and medium quality offline datasets, are shown in Table~\ref{tab: proto}. Reflecting on the empirical results discussed in our paper, we observed that the effectiveness of various mixing ratios can be influenced by the quality of the offline data. This observation highlights the importance of adaptive methods in offline-to-online reinforcement learning. Integrating ROAD with the Proto algorithm, we found that it improves upon the best static mixing ratios. This suggests that ROAD's adaptability could provide a valuable tool for optimizing data utilization in a dynamic learning environment, potentially enhancing both efficiency and policy performance under diverse conditions.

\begin{table}[ht]
\centering
\caption{The performance of ROAD with Proto and fixed mixing ratios on Halfcheetah tasks. The best scores are \textbf{bolded}. \\ \ }
\label{tab: proto}
\begin{tabular}{cccc}
\toprule
& Random    & Medium-Replay & Medium   \\ \midrule
Fixed-0.0   & 100.20    & 91.06         & 90.80    \\
Fixed-0.1   & 101.24 & 93.40 & 92.55    \\
Fixed-0.2   & 95.35     & 89.74         & 94.60 \\
Fixed-0.3   & 98.69     & 90.19         & 92.06    \\
Fixed-0.4   & 95.38     & 87.62         & 91.41    \\
Fixed-0.5   & 97.30     & 86.25         & 89.89    \\
ROAD        & \textbf{103.45} & \textbf{96.71} & \textbf{94.88} \\ \bottomrule
\bottomrule
\end{tabular}
\end{table}

% \begin{figure*}[ht]
% \centering
%   \begin{minipage}[c]{\textwidth}
%     \centering
%     \begin{subfigure}{\textwidth}
%       \centering      
%       \includegraphics[width=0.45\textwidth]{figures/ICML/MainResult/Cal-QL/CQL-all-antmaze.pdf}
%       \label{fig:second-image}
%     \end{subfigure}
%     \begin{subfigure}{\textwidth}
%       \centering      
%       \includegraphics[width=0.45\textwidth]{figures/ICML/MainResult/Cal-QL/CQL-all-adroit.pdf}
%       \label{fig:second-image}
%     \end{subfigure}
%   \end{minipage}
%   \caption{Normalized Returns of ROAD and the baseline methods with CQL \cite{kumar2020conservative} as the backbone in Antmaze and Adroit tasks.}
%   \label{fig: all CQL}
% \end{figure*}

\section{Result on RLPD: Online RL with Offline Data Replay Without Offline Pretraining} \label{appendix: result of RLPD}

As we previously discussed, recent research has uncovered that in certain settings, incorporating offline data directly into the online learning process, without any initial offline pretraining, can be highly effective. In some cases, this approach may even surpass traditional offline-to-online RL methodologies. A notable example of such an innovation is RLPD (Reinforcement Learning with Prior Data), which achieves enhanced agent performance through strategies like employing a higher update-to-data (UTD) ratio among others. In RLPD, a symmetric sampling strategy is used, meaning a 50\% mixing ratio, for blending offline and online data.

In this context, we explored integrating ROAD into the RLPD framework to attempt a more adaptive offline data replay strategy. We conducted extensive experiments across various task categories, including antmaze, locomotion, and adroit. Given the absence of offline pretraining in this setting, we designed a broader candidate set for the mixing ratio ($\Lambda = \{0.1, 0.3, 0.5, 0.7, 0.9\}$), to accommodate the unique challenges of this approach. The experimental results, depicted in Figure~\ref{fig: all RLPD}, demonstrate the outcomes of these tests. The outcomes reveal that, within this setup, ROAD consistently matches or even surpasses the optimal fixed mixing ratio across different task domains. This performance underscores ROAD's adaptability across various settings, demonstrating its potential to enhance learning outcomes without relying on offline pretraining. This exploration into the application of ROAD within the RLPD framework reveals its flexibility and effectiveness in leveraging offline data to boost online learning, even when traditional pretraining is not utilized. The broadened candidate set for mixing ratios further illustrates the importance of adaptability in the utilization of offline data, showcasing ROAD's capability to dynamically adjust to the learning environment's demands. Such results emphasize the significance of adaptive strategies in the evolving landscape of reinforcement learning, particularly in scenarios where blending offline and online learning processes can lead to improved performance and efficiency.

% \begin{figure*}[ht]
% \centering
%   \begin{minipage}[c]{\textwidth}
%     \centering
%     \begin{subfigure}{\textwidth}
%       \centering      
%       \includegraphics[width=0.32\textwidth]{figures/ICML/MainResult/RLPD/All_Antmaze_Tasks.pdf}
%       \label{fig:second-image}
%     \end{subfigure}
%     \begin{subfigure}{\textwidth}
%       \centering      
%       \includegraphics[width=0.32\textwidth]{figures/ICML/MainResult/RLPD/All_Locomotion_Tasks.pdf}
%       \label{fig:second-image}
%     \end{subfigure}
%     \begin{subfigure}{\textwidth}
%       \centering      
%       \includegraphics[width=0.32\textwidth]{figures/ICML/MainResult/RLPD/All_Adroit_Tasks.pdf}
%       \label{fig:second-image}
%     \end{subfigure}
%   \end{minipage}
%   \caption{Normalized Returns of ROAD and the fixed mixing ratio approaches with RLPD \cite{kumar2020conservative} as the backbone in Antmaze, Locomotion and Adroit tasks.}
%   \label{fig: all RLPD}
% \end{figure*}

\begin{figure*}[ht]
\centering
  \begin{minipage}[c]{\textwidth}
    \centering
    \begin{subfigure}[b]{0.32\textwidth}
      \centering      
      \includegraphics[width=\textwidth]{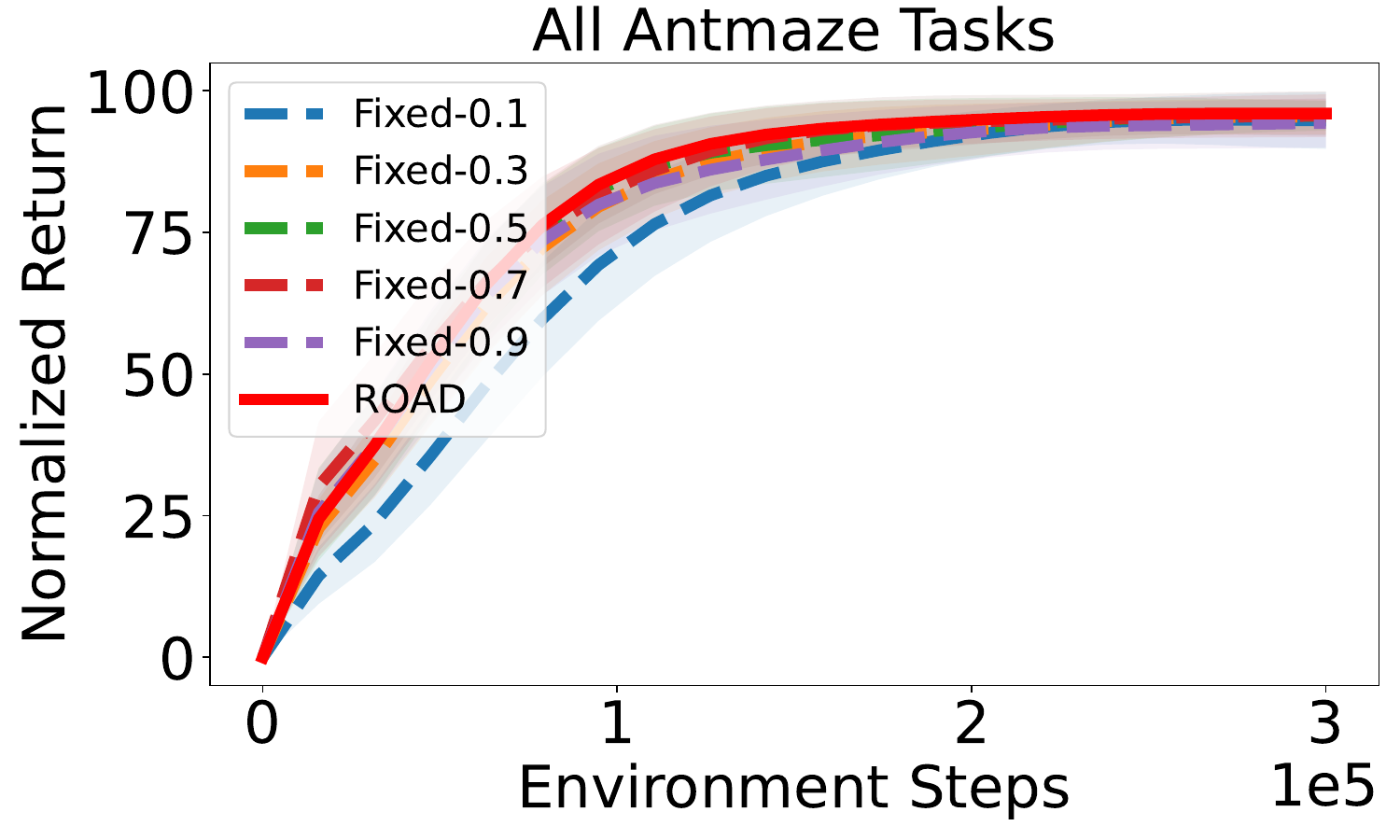}
      \caption{Antmaze Tasks}
      \label{fig:antmaze-image}
    \end{subfigure}
    \hfill % 添加空格以分隔子图
    \begin{subfigure}[b]{0.32\textwidth}
      \centering      
      \includegraphics[width=\textwidth]{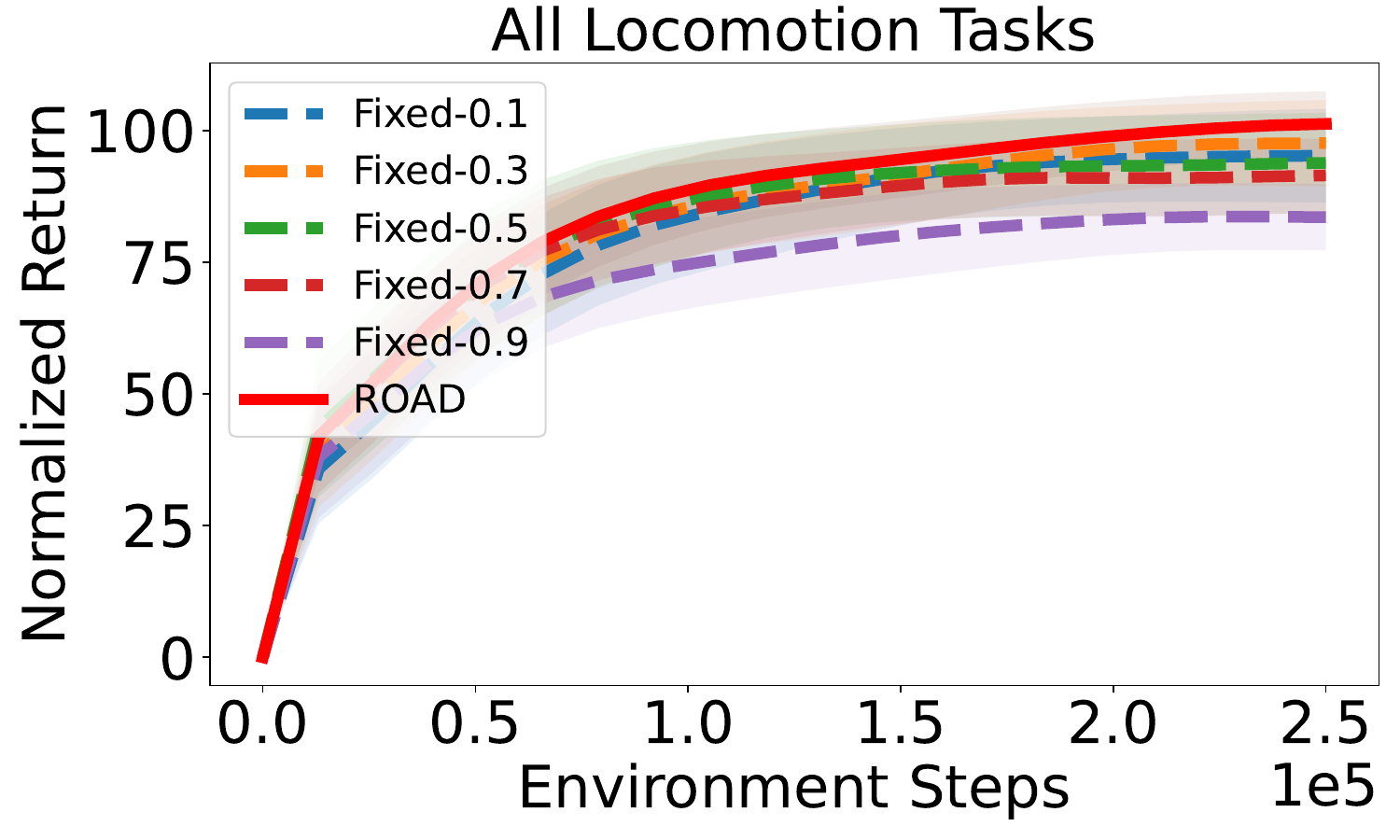}
      \caption{Locomotion Tasks}
      \label{fig:locomotion-image}
    \end{subfigure}
    \hfill % 添加空格以分隔子图
    \begin{subfigure}[b]{0.32\textwidth}
      \centering      
      \includegraphics[width=\textwidth]{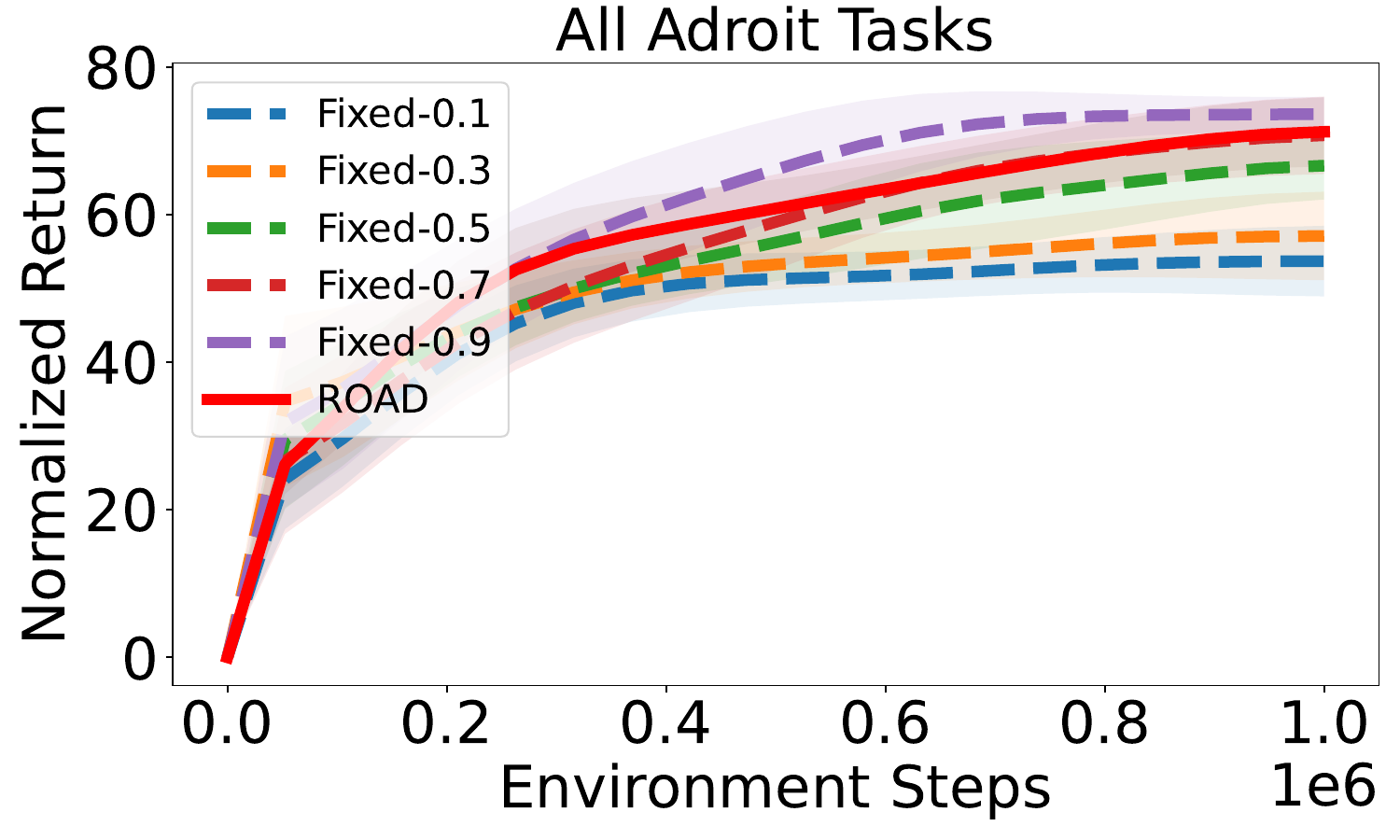}
      \caption{Adroit Tasks}
      \label{fig:adroit-image}
    \end{subfigure}
  \end{minipage}
  \caption{Normalized Returns of ROAD and the fixed mixing ratio approaches with RLPD \protect\cite{cql} as the backbone in Antmaze, Locomotion and Adroit tasks.}
  \label{fig: all RLPD}
\end{figure*}

Moreover, we also evaluated ROAD's adaptability to offline data of varying quality within the RLPD framework, particularly focusing on the locomotion task with medium, medium-replay, and random quality offline data. These assessments were compared against fixed mixing ratios. The findings, as illustrated in Figure~\ref{fig: all RLPD dataset quality}, reveal that a static mixing ratio does not consistently yield the best outcomes across different qualities of offline data. Conversely, ROAD exhibits superior performance across multiple offline data quality levels, substantiating its enhanced adaptability over fixed mixing ratios.

% \begin{figure*}[ht]
% \centering
%   \begin{minipage}[c]{\textwidth}
%     \centering
%     \begin{subfigure}{\textwidth}
%       \centering      
%       \includegraphics[width=0.32\textwidth]{figures/ICML/MainResult/RLPD/All_Locomotion_Tasks_(medium).pdf}
%       \label{fig:second-image}
%     \end{subfigure}
%     \begin{subfigure}{\textwidth}
%       \centering      
%       \includegraphics[width=0.32\textwidth]{figures/ICML/MainResult/RLPD/All_Locomotion_Tasks_(medium-replay).pdf}
%       \label{fig:second-image}
%     \end{subfigure}
%     \begin{subfigure}{\textwidth}
%       \centering      
%       \includegraphics[width=0.32\textwidth]{figures/ICML/MainResult/RLPD/All_Locomotion_Tasks_(random).pdf}
%       \label{fig:second-image}
%     \end{subfigure}
%   \end{minipage}
%   \caption{Examining ROAD under the setting of online RL with offline data without offline pretraining with various dataset qualities.}
%   \label{fig: all RLPD dataset quality}
% \end{figure*}

\begin{figure*}[ht]
\centering
  \begin{minipage}{\textwidth}
    \centering
    \begin{subfigure}[b]{0.32\textwidth}
      \centering      
      \includegraphics[width=\textwidth]{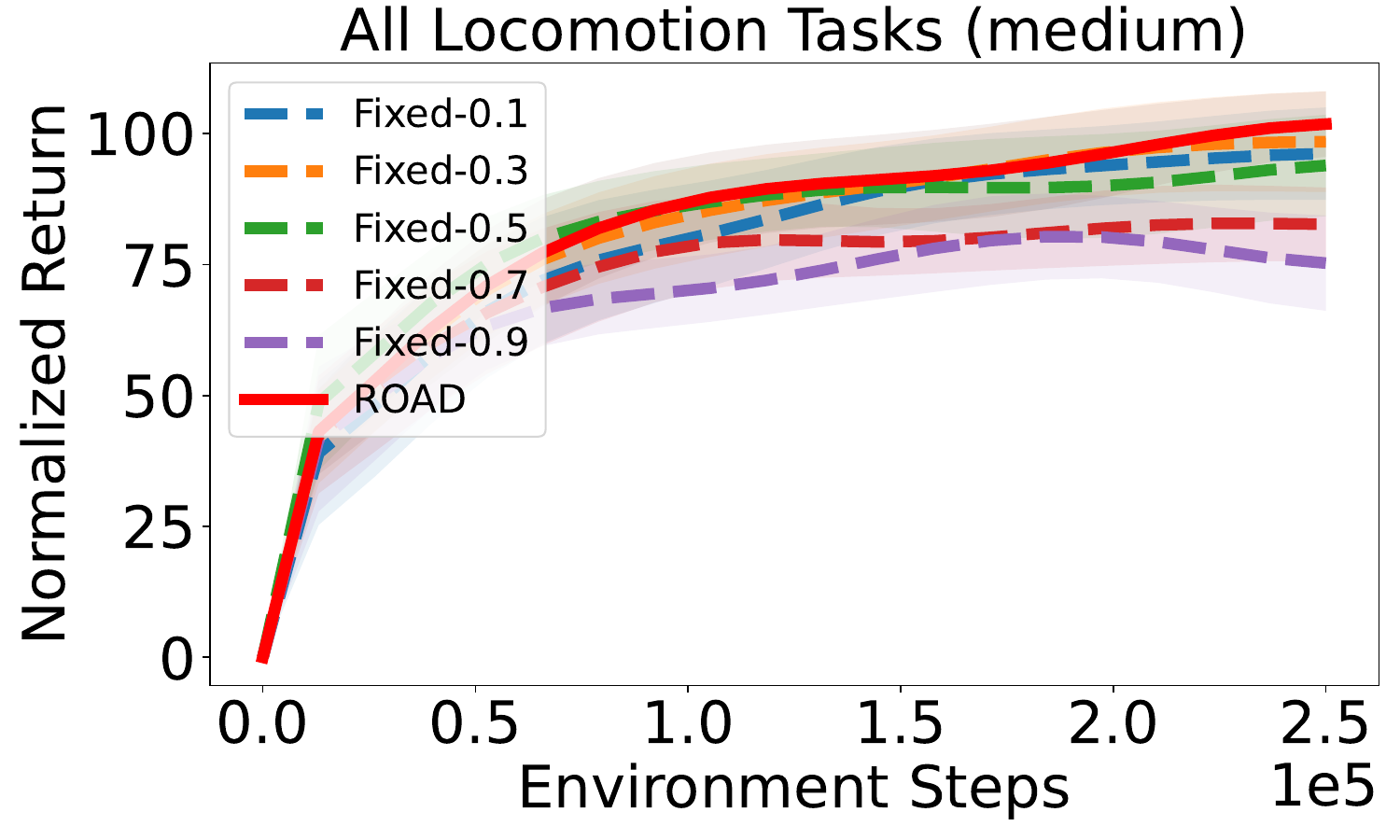}
      \caption{Medium}
      \label{fig:medium-image}
    \end{subfigure}%
    \begin{subfigure}[b]{0.32\textwidth}
      \centering      
      \includegraphics[width=\textwidth]{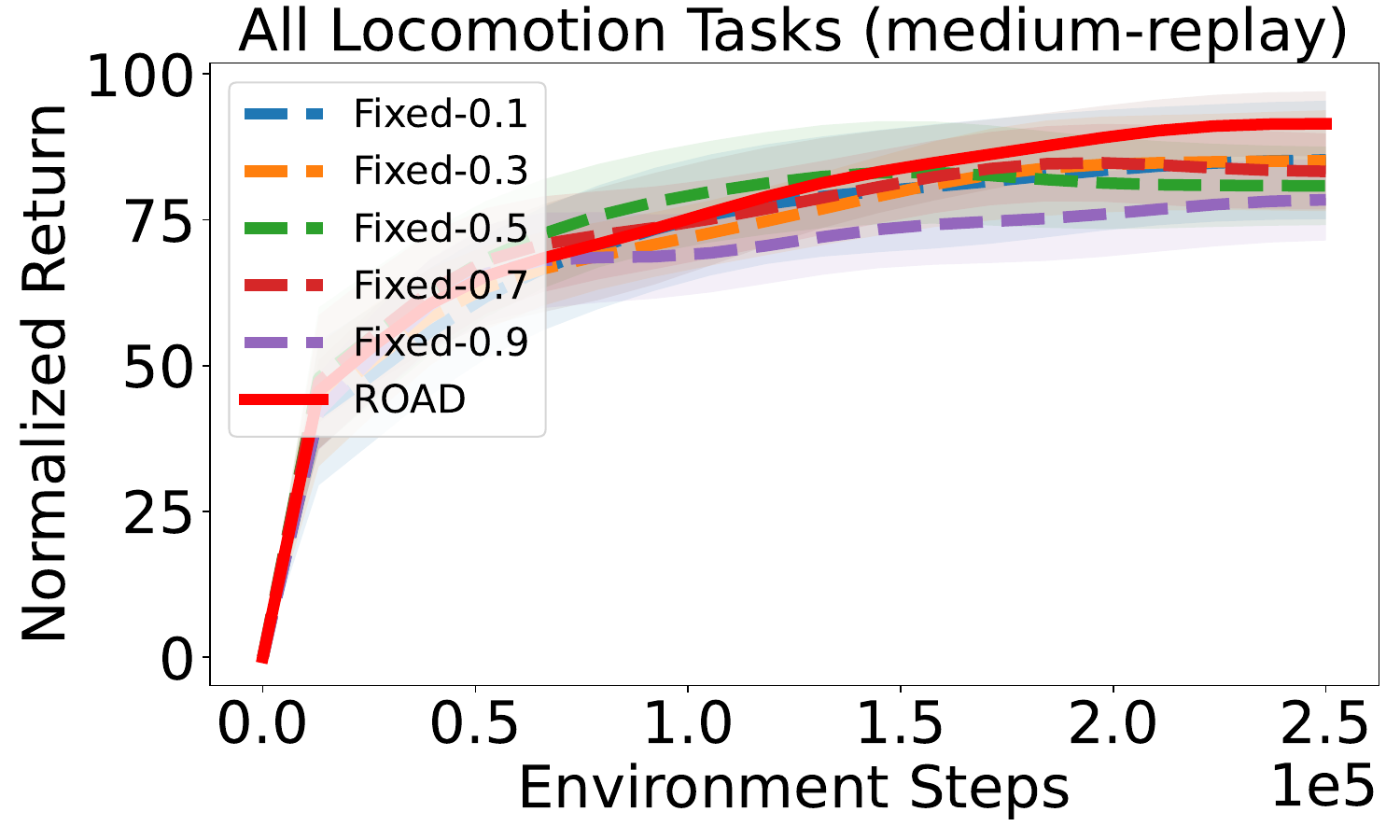}
      \caption{Medium-Replay}
      \label{fig:medium-replay-image}
    \end{subfigure}%
    \begin{subfigure}[b]{0.32\textwidth}
      \centering      
      \includegraphics[width=\textwidth]{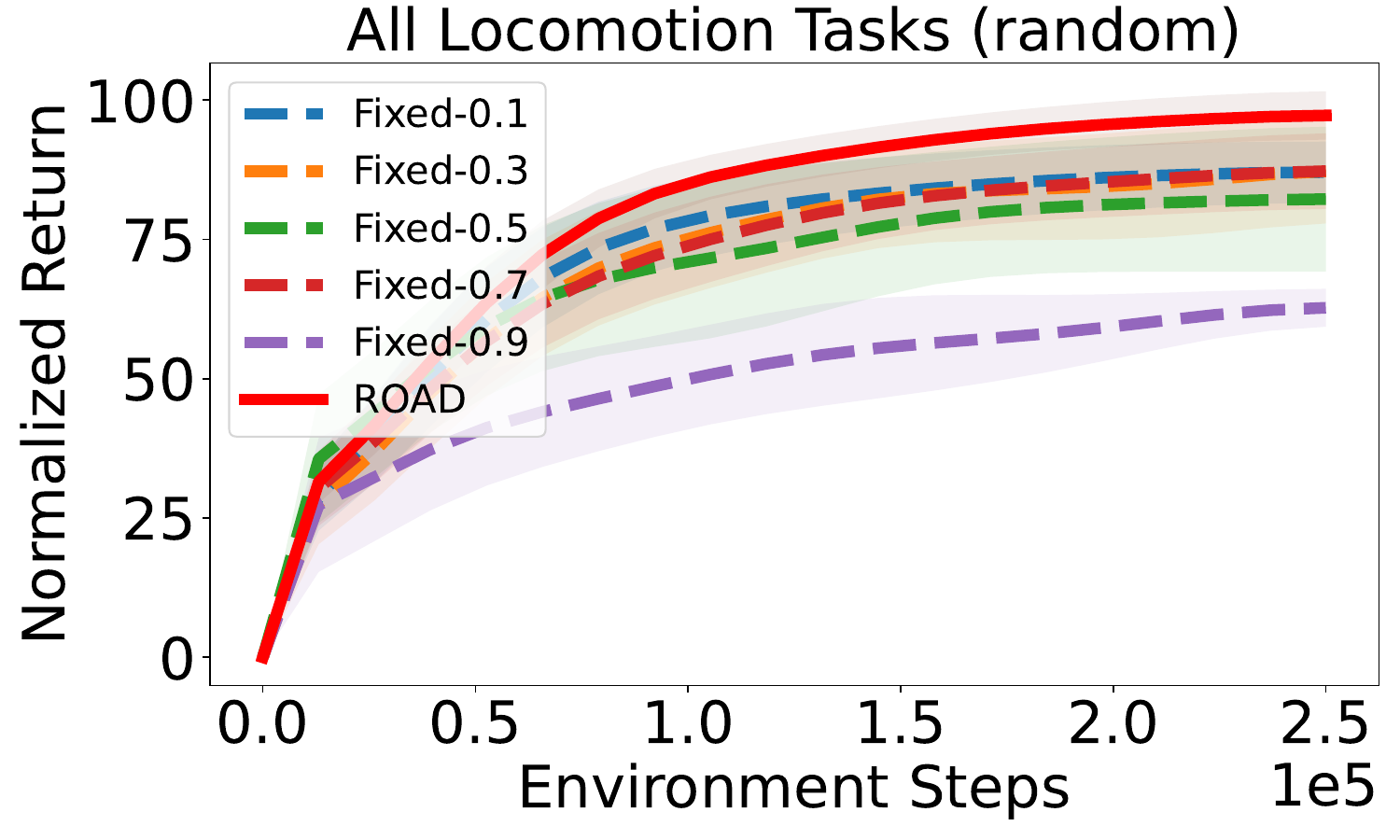}
      \caption{Random}
      \label{fig:random-image}
    \end{subfigure}
  \end{minipage}
  \caption{Examining ROAD under the setting of online RL with offline data without offline pretraining with various dataset qualities.}
  \label{fig: all RLPD dataset quality}
\end{figure*}

\newpage

\section{Hyper-Parameters Setting and Resources for Implementation} \label{appendix: params}

In our experimental setup, we employed 8 NVIDIA GeForce RTX 4090 GPUs to facilitate the main analyses presented in this paper. The computational demand of our experiments primarily stems from the offline-to-online reinforcement learning algorithms we utilized, such as Conservative Q-Learning (CQL) and Implicit Q-Learning (IQL). Notably, the introduction of ROAD did not significantly impact the overall computation cost and time cost.

The computation of expectations in Equation (\ref{eq:R_q}) is achieved through a sampling method. For this purpose, we select a batch size, $B$, and draw samples of this size from both offline and online data buffers. We then compute the empirical mean of each term using the current agent's $Q$ function, a process efficiently parallelized on GPUs that mitigates computational overhead. Additionally, the UCB algorithm updates, which involve calculating the upper confidence bounds for the arms used in our experiments, also incur minimal computational and temporal cost. Empirical evidence from our experiments shows that the execution time of ROAD, when compared to fixed mixing ratio approaches under equivalent environmental and offline dataset conditions, remains competitive. The primary computational bottleneck continues to be the interaction with the online environment, confirming that the added computational and temporal costs associated with implementing ROAD are manageable and do not detract from its efficacy.

Here we list the hyperparameters in ROAD with offline-to-online RL algorithms in Table~\ref{Tab: hyper-params2}, and for RLPD in Table~\ref{Tab: hyper-params3}.

\begin{table}[ht]
\caption{Main hyper-parameters for IQL, PEX, CQL and Cal-QL in ROAD.}
\label{Tab: hyper-params2}
\begin{center}
\begin{tabular}{cc}
\toprule
\multicolumn{1}{c}{\bf Hyper-parameter}  &\multicolumn{1}{c}{\bf Implemented Value}
\\ \midrule
batch size & 256 \\
replay buffer size & 1000000 \\
total env steps & 1000000 \\
discount & 0.99 \\
policy learning rate & 1e-4 (Cal-QL \& CQL), 3e-4 (IQL \& PEX) \\
Q function learning rate & 3e-4 \\
orthogonal init & True \\
policy arch & 256-256 \\
CQL Q function arch & 256-256-256-256 \\
CQL min Q weight & 5.0 \\
CQL target action gap & 0.8 \\
IQL expectile value & 0.9 \\
IQL inverse temperature & 10.0 \\
Exploration parameter $c$ & 2.0 \\
Sliding-window size $\tau$ & $\tau=t$ for Antmaze, $\tau = 1000$ for other tasks \\
\bottomrule

\end{tabular}
\end{center}
\end{table}

\begin{table}[ht]
\caption{Main hyper-parameters for RLPD in ROAD.}
\label{Tab: hyper-params3}
\begin{center}
\begin{tabular}{cc}
\toprule
\multicolumn{1}{c}{\bf Hyper-parameter}  &\multicolumn{1}{c}{\bf Implemented Value}
\\ \midrule
batch size & 128 \\
replay buffer size & 1000000 \\
total env steps & 1000000 \\
discount & 0.99 \\
policy learning rate & 3e-4 \\
ensemble size & 10 \\
RLPD network width & 256 \\
UTD ratio & 20 \\
Exploration parameter $c$ & 2.0 \\
Sliding-window size $\tau$ & $\tau=t$ \\
\bottomrule

\end{tabular}
\end{center}
\end{table}

% \include{sections/7-appendix}
% \include{sections/8-appendix2}

% \section*{Ethical Statement}

% There are no ethical issues.

% \section*{Acknowledgments}

% The preparation of these instructions and the \LaTeX{} and Bib\TeX{}
% files that implement them was supported by Schlumberger Palo Alto
% Research, AT\&T Bell Laboratories, and Morgan Kaufmann Publishers.
% Preparation of the Microsoft Word file was supported by IJCAI.  An
% early version of this document was created by Shirley Jowell and Peter
% F. Patel-Schneider.  It was subsequently modified by Jennifer
% Ballentine, Thomas Dean, Bernhard Nebel, Daniel Pagenstecher,
% Kurt Steinkraus, Toby Walsh, Carles Sierra, Marc Pujol-Gonzalez,
% Francisco Cruz-Mencia and Edith Elkind.

\end{document}